\documentclass[11pt,a4paper]{article}
\usepackage[utf8]{inputenc}
\usepackage[T1]{fontenc}
\usepackage{mathpazo}
\usepackage{graphicx, graphics, epsfig, color}
\usepackage{booktabs}
\usepackage{subcaption}
\usepackage{amsmath,amssymb,amsthm}
\usepackage{tikz, pgfplots}
\usetikzlibrary{plotmarks,spy}
\pgfplotsset{compat=newest}
\usepackage[margin=3cm]{geometry}
\usepackage{hyperref}
\hypersetup{colorlinks=false}

\newcommand{\footremember}[2]{%
    \footnote{#2}
    \newcounter{#1}
    \setcounter{#1}{\value{footnote}}%
}
\newcommand{\footrecall}[1]{%
    \footnotemark[\value{#1}]%
} 

\title{Inner-product Kernels are Asymptotically Equivalent to Binary Discrete Kernels}

\author{Zhenyu Liao\footremember{cs}{Laboratoire des Signaux et Syst{\`e}me, CentraleSup{\'e}lec, Universit{\'e} Paris-Saclay, France.} \footremember{gipsa}{G-STATS Data Science Chair, GIPSA-lab, University Grenoble Alpes, France.} \and \and Romain Couillet\footrecall{gipsa} \footrecall{cs}}

\date{}

\DeclareMathOperator{\tr}{ {\rm tr} }
\DeclareMathOperator{\diag}{ {\rm diag} }

\DeclareMathOperator{\erf}{ {\rm erf} }
\newcommand{\RR}{\mathbb{R}}
\newcommand{\CC}{\mathbb{C}}
\newcommand{\T}{ {\sf T} }
\newcommand{\EE}{\mathbb{E}}
\newcommand{\Var}{{\rm Var}}
\newcommand{\NN}{\mathcal{N}}

\newcommand{\zo}{{\mathbf{0}}}
\newcommand{\E}{{\mathbf{E}}}
\newcommand{\K}{{\mathbf{K}}}
\newcommand{\x}{{\mathbf{x}}}
\newcommand{\z}{{\mathbf{z}}}
\newcommand{\X}{{\mathbf{X}}}
\newcommand{\Y}{{\mathbf{Y}}}
\newcommand{\Z}{{\mathbf{Z}}}
\newcommand{\M}{{\mathbf{M}}}
\newcommand{\A}{{\mathbf{A}}}
\newcommand{\B}{{\mathbf{B}}}
\newcommand{\F}{{\mathbf{F}}}
\newcommand{\N}{{\mathbf{N}}}
\newcommand{\W}{{\mathbf{W}}}

\newcommand{\e}{{\mathbf{e}}}
\newcommand{\y}{{\mathbf{y}}}
\newcommand{\one}{{\mathbf{1}}}

\newcommand{\U}{{\mathbf{U}}}

\newcommand{\bPhi}{ \boldsymbol{\Phi} }
\newcommand{\bmu}{ \boldsymbol{\mu} }

\newcommand{\C}{{\mathbf{C}}}
\newcommand{\I}{{\mathbf{I}}}
\newcommand{\J}{{\mathbf{J}}}

\renewcommand{\v}{{\mathbf{v}}}
\renewcommand{\a}{{\mathbf{a}}}
\renewcommand{\b}{{\mathbf{b}}}

\renewcommand{\j}{{\mathbf{j}}}

\renewcommand{\S}{{\mathbf{S}}}
\renewcommand{\L}{{\mathbf{L}}}

\DeclareMathOperator{\sign}{{\rm sign}}

\definecolor{RED}{rgb}{0.7,0,0}
\definecolor{BLUE}{rgb}{0,0,0.69}
\definecolor{GREEN}{rgb}{0,0.65,0}
\definecolor{PURPLE}{rgb}{0.69,0,0.8}

\newtheorem{Theorem}{Theorem}
\newtheorem{Proposition}{Proposition}
\newtheorem{Lemma}{Lemma}
\newtheorem{Remark}{Remark}

\newtheorem{Assumption}{Assumption}

\begin{document}

\maketitle

\begin{abstract}
This article investigates the eigenspectrum of the inner product-type kernel matrix \( \sqrt{p}\K=\{f(\x_i^\T \x_j/\sqrt{p})\}_{i,j=1}^n \) under a binary mixture model in the high dimensional regime where the number of data \(n\) and their dimension \(p\) are both large and comparable. Based on recent advances in random matrix theory, we show that, for a wide range of nonlinear functions \(f\), the eigenspectrum behavior is asymptotically equivalent to that of an (at most) cubic function. This sheds new light on the understanding of nonlinearity in large dimensional problems. As a byproduct, we propose a simple function prototype valued in \( (-1,0,1) \) that, while reducing both storage memory and running time, achieves the same (asymptotic) classification performance as any arbitrary function \(f\).
\end{abstract}

\section{Introduction}

Multivariate mixture models, especially Gaussian mixtures, play a fundamental role in statistics and have received significant research attention in the machine learning community \cite{kalai2016disentangling,ashtiani2018nearly}, mainly due to the convenient statistical properties of Gaussian and sub-Gaussian distributions. More generic mixture models, however, are somehow less covered.

On the other hand, in the study of large random matrices, one is able to reach in some cases ``universal'' results in the sense that the (asymptotic) statistical behavior of the object of interest is \emph{independent} of the underlying distribution of the random matrix. Intuitively speaking, the ``squared'' number of degrees of freedom in large matrices (e.g., sample covariance matrices \(\frac1n \X \X^\T \in \RR^{p \times p}\) based on \(n\) observations of dimension \(p\) arranged in the columns of \(\X\)) and their \emph{independent} interactions induce fast versions of central limit theorems irrespective of the data distribution, resulting in universal statistical results.

In this paper, we consider the eigenspectrum behavior of the inner product ``properly scaled'' (see details below) kernel matrix \(\K_{ij} = f(\x_i^\T \x_j/\sqrt p) /\sqrt p\), for \(n\) data \(\x_i\in\RR^p\) arising from an affine transformation of i.i.d.\@ random vectors with zero mean, unit covariance and bounded higher order moments. Under this setting and some mild regularity condition for the nonlinear function \(f\), the spectrum of \(\K\) can be shown to only depend on \emph{three} parameters of \(f\), in the regime where \(n,p\) are both large and comparable.

The theoretical study of the eigenspectrum of large random matrices serves as a basis to understand many practical statistical learning algorithms, among which are kernel spectral clustering \cite{ng2002spectral} or sparse principle component analysis (PCA) \cite{johnstone2009consistency}. In the large \(n,p\) regime, various types of ``random kernel matrices'' have been studied from a spectral random matrix viewpoint. In \cite{el2010spectrum} the authors considered kernel matrices based on the Euclidean distance \(f(\| \x_i - \x_j \|^2/p)\) or the inner product \(f(\x_i^\T \x_j/p)\) between data vectors, and study the eigenvalue distribution by essentially ``linearizing'' the nonlinear function \(f\) via a Taylor expansion, which naturally demands \(f\) to be locally smooth. Later in \cite{couillet2016kernel} the authors followed the same technical approach and considered a more involved Gaussian mixture model for the \(\x_i\), providing rich insights into the impact of nonlinearity in kernel spectral clustering application.

Nonetheless, these results are in essence built upon a local expansion of the nonlinear function \(f\), which follows from the ``concentration'' of the similarity measures \(\| \x_i - \x_j \|^2/p\) or \(\x_i^\T \x_j/p\) around a \emph{single} value of the smooth domain of \(f\), therefore disregarding most of the domain of \(f\). In this article, following \cite{cheng2013spectrum,do2013spectrum}, we study the inner product kernel \(f(\x_i^\T \x_j/\sqrt p) /\sqrt p\) which avoids the concentration effects with the more natural \(\sqrt{p}\) normalization. With the flexible tool of orthogonal polynomials, we are able to prove universal results which solely depend on the first two order moments of the data distribution and allow for nonlinear functions \(f\) that need not even be differentiable. As a practical outcome of our theoretical results, we propose an extremely simple piecewise constant function which is spectrally equivalent and thus performs equally well as arbitrarily complex functions \(f\), while inducing enormous gains in both storage and computational complexity. 

\medskip

The remainder of this article is organized as follows. In Section~\ref{sec:model} we introduce the object of interest together with necessary assumptions to work along with. Our main theoretical findings are presented in Section~\ref{sec:main} with intuitive ideas, while detailed proofs are deferred to the supplementary material due to space limitation. In Section~\ref{sec:practice} we discuss the practical consequence of our theoretical findings and propose our piecewise constant function prototype which works in a universal manner for kernel spectrum-based applications, for the system model under consideration. The article closes with concluding remarks and envisioned extensions in Section~\ref{sec:conclusion}.

\medskip

\emph{Notations}: Boldface lowercase (uppercase) characters stand for vectors (matrices). The notation \((\cdot)^\T\) denotes the transpose operator. The norm \(\|\cdot\|\) is the Euclidean norm for vectors and the operator norm for matrices, and we denote \(\| \A \|_\infty = \max_{i,j} |\A_{ij}|\) as well as \(\|\cdot\|_F\) the Frobenius norm: \(\|\A\|_F^2=\tr(\A\A^\T)\). \(\xi\) is often used to denote standard Gaussian random variable, i.e., \(\xi \sim \NN(0,1)\). 

\section{System model and preliminaries}
\label{sec:model}

\subsection{Basic setting}

Let \(\x_1, \ldots, \x_n \in \RR^p\) be \(n\) feature vectors drawn independently from the following two-class (\(\mathcal{C}_1\) and \(\mathcal{C}_2\)) mixture model:
\begin{equation}
    \label{eq:mixture}
  %\mathcal{C}_1: \x = \bmu_1 + (\I_p + \E_1)^{\frac12} \z; \quad \mathcal{C}_2: \x = \bmu_2 + (\I_p + \E_2)^{\frac12} \z,
  \begin{cases}
    \mathcal{C}_1: &\x = \bmu_1 + (\I_p + \E_1)^{\frac12} \z \\ 
    %%%
    \mathcal{C}_2: &\x = \bmu_2 + (\I_p + \E_2)^{\frac12} \z
  \end{cases}
\end{equation}
each having cardinality \(n/2\),\footnote{We restrict ourselves to binary classification for readability, but the presented results easily extend to a multi-class setting. %by rewriting \(\M\), \(\mathbf{T}\) and \(\S\) as their multi-class counterparts. {\RED *** pas définis ici ***}
%Also, while propositions hold for unbalanced class priors, some discussions become invalid, mainly due to the possible presence of non-informative isolated eigenvalue found in \cite{fan2019spectral}. This is however beyond the scope of this article. {\RED *** Est-ce que cette dernière phrase n'est pas un peu troublante? Faut-il vraiment la mettre? ***} 
} for some deterministic \(\bmu_a \in \RR^p\), \(\E_a \in \RR^{p \times p}\), \(a=1,2\) and random vector \(\z \in \RR^p\) having i.i.d.\@ entries of zero mean, unit variance and bounded moments. %, i.e., for \(k \ge 1\) there exists \(C_k\) independent of \(n,p\) such that \(\EE |[\z]_i|^k \le C_k\).
To ensure that the information of \(\bmu_a, \E_a\) is neither (asymptotically) too simple nor impossible to be extracted from the noisy features\footnote{We refer the readers to Section~\ref{sec:sm-non-trivial} in Supplementary Material for a more detailed discussion on this point.}, we work (as in \cite{couillet2018classification}) under the following assumption.
%and some constant \(\sigma >0\)
\begin{Assumption}[Non-trivial classification]\label{ass:growth-rate}
As \(n \to \infty\), we have for \(a\in\{1,2\}\)
\begin{enumerate}
  \item[(i)] \(p/n = c \to \bar c \in (0,\infty)\),
  \item[(ii)] \(\| \bmu_a \| = O(1)\), \(\| \E_a \| = O(p^{-1/4})\), \(|\tr(\E_a)| = O( \sqrt p )\) and \(\| \E_a \|_F^2 = O(\sqrt p)\). %$\tr(\E_a^2) = O(1)$.
\end{enumerate}
\end{Assumption}
% The feature vector \(\x_i \in \RR^p\) is said to belong to the class \(\mathcal{C}_a\) if
% \[
%     \x_i = \bmu_a + \sigma \z_i, \quad \z_i \sim \mathcal{N}(\mathbf{0}, \I_p),
% \]
% for some mean vector \(\bmu_a \in \RR^p\) and constant \(\sigma > 0\). We denote the feature matrix \(\X = [\x_1, \ldots, \x_n] \in \RR^{p \times n}\) by cascading all \(\x_i\) as its column vectors. Note that the above model can be rewritten in matrix model as
% \[
%     \X = \Z + \M \J^\T
% \]
% with

Following \cite{el2010information,cheng2013spectrum} we consider the following nonlinear random inner-product matrix
\begin{equation}\label{eq:def-K}
    \K = \left\{\delta_{i \neq j} f (\x_i^\T \x_j/\sqrt p)/\sqrt p\right\}_{i,j=1}^n
\end{equation}
for function \(f: \RR \mapsto \RR\) satisfying some regularity conditions (detailed later in Assumption~\ref{ass:poly}). As in \cite{el2010information,cheng2013spectrum}, the diagonal elements \(f(\x_i^\T\x_i/\sqrt{p})\) have been discarded. Indeed, under Assumption~\ref{ass:growth-rate}, \(\x_i^\T\x_i/\sqrt{p}=O(\sqrt{p})\) which is an ``improper scaling'' for the evaluation by \(f\) (unlike \(\x_i^\T\x_j/\sqrt{p}\) which properly scales as \(O(1)\) for all \(i\neq j\)).
%those  from \(\K\) removed since (i) it merely involves the norms \(\| \x_1 \|^2, \ldots, \| \x_n \|^2\), based upon which the classification becomes trivial and there is no necessity to kernel in this case {\RED *** cet argument n'est pas bon! pourquoi se priver de cette information si elle est utile? ***}, and (ii) it possesses entries of the form \(f(\| \x_i\|^2/\sqrt p)/\sqrt p\) that may not converge as \(p \to \infty\), for example in the case of \(f(x) = x^2\) {\RED *** ce n'est pas exact non plus: ce qui diverge c'est \(\|\x_i\|^2/\sqrt{p}\) et donc \(f\) est asymptotiquement évaluée en \(+\infty\), c'est ca le problème. Il faut recrire cet argument. ***}. 

In the absence of discriminative information (null model), i.e., if \(\bmu_a=\zo\) and \(\E_a=\zo\) for \(a=1,2\), we write \(\K=\K_N\) with
\begin{equation}\label{eq:def-K-n}
    [\K_N]_{ij} = \delta_{i \neq j} f ( \z_i^\T \z_j/\sqrt p)/\sqrt p.
\end{equation}
Letting \(\xi_p \equiv \z_i^\T \z_j/\sqrt p\), by the central limit theorem, \(\xi_p\overset{\mathcal L}{\longrightarrow}\NN(0,1)\) as \(p \to \infty\). As such, the \([\K_N]_{ij}\), \(1\leq i\neq j\leq n\), asymptotically behave like a family of \emph{dependent} standard Gaussian variables to which \(f\) is applied. In order to analyze the joint behavior of this family, we shall exploit some useful concepts of the theory of orthogonal polynomials and, in particular, of the class of Hermite polynomials defined with respect to the standard Gaussian distribution
\cite{abramowitz1965handbook,andrews2000special}. 
%For more details we refer the readers to \cite{szeg1939orthogonal} and the references therein.

%For a real probability measure \(\mu\), the set of orthogonal polynomials with respect to the scalar product \(\left<f,g\right>=\int fg d\mu\) is the family \(\{P_l(x), l=0,1,\ldots\}\) obtained using the Gram-Schmidt decomposition on the monomials \(\{1,x,x^2,\ldots\}\), i.e., \(P_0(x) = 1\) {\RED *** \(P_1(x) = x\) n'est pas vrai Zhenyu. \(P_1\) vaut \(x\) pour \(\mu\) de moyenne nulle mais pas dans le cas général ***} and, for \(l\geq 1\), \(P_l(x)\) is the Gram-Schmidt polynomial of degree \(l\) obtained from \(x^l\) and orthogonal to \(P_0,\ldots,P_{l-1}\). 
%Then, one has by orthogonality that \(\int_\RR P_{l_1}(x) P_{l_2}(x) d \mu(x) = 1\) for \(l_1 = l_2\) and \(0\) otherwise.
% \[
%     \int_\RR P_{l_1}(x) P_{l_2}(x) d \mu(x) = \begin{cases} 0, \quad l_1 \neq l_2, \\ 1, \quad l_1=l_2. \end{cases}
% \]
%Besides, for any function \(f\) square-integrable with respect to \(\mu\), one can expand \(f\) as
%\[
%    f(x) = \sum_{l=0}^\infty a_l P_l(x), \quad \text{with}\ a_l = \int_\RR f(x) P_l(x) d \mu(x).
%\]

\subsection{The Orthogonal Polynomial Framework}

For a real probability measure \(\mu\), we denote the set of orthogonal polynomials with respect to the scalar product \(\left<f,g\right>=\int fg d\mu\) as \(\{P_l(x), l=0,1,\ldots\}\), obtained from the Gram-Schmidt procedure on the monomials \(\{1,x,x^2,\ldots\}\) such that \(P_0(x)=1\), \(P_{l}\) is of degree \(l\) and \(\left< P_{l_1}, P_{l_2}\right> = \delta_{l_1-l_2}\). By the Riesz-Fisher theorem \cite[Theorem~11.43]{rudin1964principles}, for any function \(f \in L^2(\mu)\), the set of squared integrable functions with respect to \(\left<\cdot,\cdot\right>\), % that belongs to the span of \(\{P_l, l \ge 0\}\),
%\footnote{Note that, the set of orthogonal polynomials with respect to \(\mu\) does not always form a complete basis in \(L^2(\mu)\), this however holds for a fairly large class of distributions, in particular those with sub-Gaussian tails as pointed out in \cite{do2013spectrum}.} 
one can formally expand \(f\) as
\begin{equation}\label{eq:def-expansion}
    f(x) \sim \sum_{l=0}^\infty a_l P_l(x), \quad a_l = \int_\RR f(x) P_l(x) d \mu(x)
\end{equation}
where ``\(f\sim \sum_{l=0}^\infty P_l\)'' indicates that \(\|f-\sum_{l=0}^N P_l\|\to 0\) as \(N\to\infty\) (and \(\|f\|^2=\left<f,f\right>\)).

% In particular, since that we are interested in the Gaussian measure \(\mu(dx) = \frac1{\sqrt{2\pi}} e^{-x^2/2}\), the (possibly) most well known choice in this case is the normalized Hermite polynomials, the explicit formula of which is given by \cite{abramowitz1964handbook}
% \[
%     H_l(x) = \frac1{\sqrt{l!}} \sum_{k=0}^{\lfloor l/2 \rfloor} \left( -\frac12 \right)^k \frac{x^{l-2k}}{ k! (l-2k)! } , \quad l = 0,1,\ldots
% \]
% {\RED space and orthogonal basis of this space?} As a matter of fact, Lemma~4.1 in \cite{cheng2013spectrum} allows one to essentially ``expand'' \(f(\xi)\) with normalized Hermite polynomials if only the first finite-many terms are considered, since the coefficients of \(P_l(x)\) with respect to the probability density of \(\xi\) converge to those of \(H_l(x)\). 

\medskip

To investigate the asymptotic behavior of \(\K\) and \(\K_N\) as \(n,p \to \infty\), we position ourselves under the following technical assumption.
\begin{Assumption}\label{ass:poly}
For each \(p\), let \(\xi_p = \z_i^\T \z_j/\sqrt p\) and let \(\{P_{l,p}(x),\, l \ge 0 \}\) be the set of orthonormal polynomials with respect to the probability measure \(\mu_p\) of \(\xi_p\).\footnote{Note that \(\mu_p\) is merely standard Gaussian in the large \(p\) limit.} For \(f\in L^2(\mu_p)\) for each \(p\), i.e., 
\[
  f(x) \sim \sum_{l=0}^\infty a_{l,p} P_{l,p}(x)
\]
for \(a_{l,p}\) defined in \eqref{eq:def-expansion}, we demand that
\begin{enumerate}
  %in \(L^2(\mu_p)\) 
  \item[(i)] \(\sum_{l=0}^\infty a_{l,p} P_{l,p}(x) \mu_p(dx)\) converges in \(L^2(\mu_p)\) to \(f(x)\) uniformly over large \(p\), i.e., for any \(\epsilon > 0\) there exists \(L\) such that for all \(p\) large, %the \(L^2\) norm with respect to \(\mu_p\) satisfies
  \[
    \Big\| f - \sum_{l=0}^L a_{l,p} P_{l,p} \Big\|^2_{L^2(\mu_p)} = \sum_{l=L+1}^\infty |a_{l,p}|^2 \le \epsilon,
  \]
  \item[(ii)] as \(p \to \infty\), \(\sum_{l=1}^\infty |a_{l,p}|^2 \to \nu \in [0, \infty)\). Moreover, for \(l=0,1,2\), \(a_{l,p}\) converges and we denote \(a_0\), \(a_1\) and \(a_2\) their limits, respectively.
  \item[(iii)] \(a_0=0\).
\end{enumerate}
\end{Assumption}

Since \(\xi_p \to \NN(0,1)\), the limiting parameters \(a_0, a_1, a_2\) and \(\nu\) are simply (generalized) moments of the standard Gaussian measure involving \(f\). Precisely,
\[
  a_0 = \EE[f(\xi)], \ a_1 = \EE[\xi f(\xi)], \ a_2 = \frac{\EE[(\xi^2-1) f(\xi)]}{\sqrt 2} = \frac{\EE[\xi^2 f(\xi)] - a_0}{\sqrt 2}, \ \nu = \Var[f(\xi)] \ge a_1^2 + a_2^2
\]
for \(\xi \sim \NN(0,1)\). These parameters are of crucial significance in determining the eigenspectrum behavior of \(\K\). Note that \(a_0\) will not affect the classification performance, as described below.

\begin{Remark}[On \(a_0\)]\label{rem:a_0}
%(i.e., \(|\mathcal C_1|=|\mathcal C_2|\))
In the present case of balanced mixtures (equal cardinalities for \(\mathcal{C}_1\) and \(\mathcal{C}_2\)), \(a_0\) contributes to the polynomial expansion of \(\K_N\) (and \(\K\)) as a non-informative rank-1 perturbation of the form \(a_0(\one_n \one_n^\T - \I_n)/\sqrt p\). Since \(\one_n\) is orthogonal to the ``class-information vector'' \([\one_{n/2},-\one_{n/2}]\), its presence does not impact the classification performance.\footnote{If mixtures are unbalanced, the vector \(\one_n\) may tend to ``pull'' eigenvectors aligned to \([\one_{n_1},-\one_{n_2}]\), with \(n_i\) the cardinality in \(\mathcal{C}_i\), so away from purely noisy eigenvectors and thereby impacting classification performance. See \cite{couillet2016kernel} for similar considerations.}
\end{Remark}

\subsection{Limiting spectrum of \texorpdfstring{\(\K_N\)}{KN}}

%To characterize the eigenspectrum of large random matrices, the Stieltjes transform is a powerful tool to study the limiting eigenvalue measure \(\mu\) of interest, if exists, defined as \(m(z) = \int_\RR \frac1{t-z} d \mu(t)\), for \(z \in \CC\) not in the support of \(\mu\). One can recover all points of continuity of \(\mu\) via the inverse formula \(\mu(x) = \frac1\pi \lim_{y\to 0+} \int_\RR \Im[m(t+ \imath y)] dt\). It turns out that the Stieltjes transform is a very efficient, is not the only, way to study the spectral behavior of large dimensional random matrices.

It was shown in \cite{cheng2013spectrum,do2013spectrum} that, for independent \(\z_i\)'s with independent entries, the \emph{empirical spectral measure} \(\mathcal L_n=\frac1n\sum_{i=1}^n\delta_{\lambda_i(\K_N)}\) of the null model \(\K_N\) has an asymptotically deterministic behavior as \(n,p \to \infty\) with \(p/n\to \bar c\in(0,\infty)\). % under the form of the associated Stieltjes transform.
\begin{Theorem}[\cite{cheng2013spectrum,do2013spectrum}]\label{theo:lsd}
Let \(p/n = c \to \bar c \in (0,\infty)\) and Assumption~\ref{ass:poly} hold. % for \(\z_i\sim\mathcal N(0,\I_p)\). 
Then, the empirical spectral measure \(\mathcal L_n\)
%\footnote{The empirical spectral measure \(\mu_\A\) of a random matrix \(\A \in \RR^{n \times n}\) is the defined as the normalized counting measure of the eigenvalues \(\lambda_1(\A), \ldots, \lambda_n(\A)\) of \(\A\): \(\mu_\A(x) = \frac1n \sum_{i=1}^n \delta_{\lambda_i(\A)} (x)\). If the random \(\mu_\A\) converges, as \(n\to\infty\), to some non-random limit \(\mu\), we call \(\mu\) the limiting spectral measure of \(\A\).} 
of \(\K_N\) defined in \eqref{eq:def-K-n} converges weakly and almost surely to a probability measure \(\mathcal L\). The latter is uniquely defined through its Stieltjes transform \(m:\CC^+\to\CC^+\), \(z\mapsto \int (t-z)^{-1}\mathcal L(dt)\), given as the unique solution in \(\CC^+\) of the (cubic) equation\footnote{\(\CC^+\equiv \{z\in\CC,~\Im[z]>0\}\). We also recall that, for \(m(z)\) the Stieltjes transform of a measure \(\mu\), \(\mu\) can be obtained from \(m(z)\) via \(\mu([a,b])=\lim_{\epsilon\downarrow 0}\frac1\pi\int_a^b\Im[m(x+\imath \epsilon)]dx\) for all \(a<b\) continuity points of \(\mu\).}
\[
    -\frac1{m(z)} = z + \frac{a_1^2 m(z)}{ c + a_1 m(z) } + \frac{\nu - a_1^2}{c} m(z).
\]
\end{Theorem}
%This result, originally proved for the Gaussian case \(\z_i \sim \NN(\zo, \I_p)\) in \cite{cheng2013spectrum}, is then generalized in \cite{do2013spectrum} to any random vector \(\z\) having i.i.d.\@ entries of zero mean, unit variance and finite (absolute) moments, i.e., \(\EE |z_i|^k \le C_k\) for all \(k \ge 1\) and \(C_k\) independent of \(n,p\). 

Theorem~\ref{theo:lsd} is ``universal'' with respect to the law of the (independent) entries of \(\z_i\). While universality is classical in random matrix results, with mostly first and second order statistics involved, the present universality result is much less obvious since (i) the nonlinear application \(f(\x_i^\T\x_j/\sqrt{p})\) depends in an intricate manner on all moments of \(\x_i^\T\x_j\) and (ii) the entries of \(\K_N\) are strongly dependent. In essence, universality still holds here because the convergence speed to Gaussian of \(\x_i^\T\x_j/\sqrt{p}\) is sufficiently fast to compensate the residual impact of higher order moments in the spectrum of \(\K_N\).

%As we shall see later, this ``universality'' holds not only for the (limiting) spectral measure of the null model \(\K_N\), but also for those isolated and informative eigenvalue/vector pairs that carry the class structural information of the data.
As an illustration, Figure~\ref{fig:eigs-K-N} compares the empirical spectral measure of \(\K_N\) to the limiting measure \(\mu\) of Theorem~\ref{theo:lsd}.\footnote{For all figures in this article, the eigenvalues (that produce the empirical histograms) as well as the associated eigenvectors are computed by MATLAB's \texttt{eig(s)} function and correspond to a \emph{single} realization of the random kernel matrix \(\K\) or \(\K_N\).}

\medskip

From a technical viewpoint, the objective of the article is to go beyond the null model described in Theorem~\ref{theo:lsd} by providing a tractable \emph{random matrix equivalent} \(\tilde\K\) for the kernel matrix \(\K\), in the sense that \(\|\K-\tilde\K\|\to 0\) almost surely in operator norm, as \(n,p \to \infty\). This convergence allows one to identify the eigenvalues and isolated eigenvectors (that can be used for spectral clustering purpose) of \(\K\) by means of those of \(\tilde\K\), see for instance \cite[Corollary~4.3.15]{horn2012matrix}. More importantly, while not visible from the expression of \(\K\), the impact of the mixture model (\(\bmu_1,\bmu_2,\E_1,\E_2\)) on \(\K\) is readily accessed from \(\tilde\K\) and easily related to the Hermite coefficients ( \(a_1,a_2,\nu\) ) of \(f\). This allows us to further investigate how the choice of \(f\) impacts the asymptotically feasibility and efficiency of spectral clustering from the top eigenvectors of \(\K\).

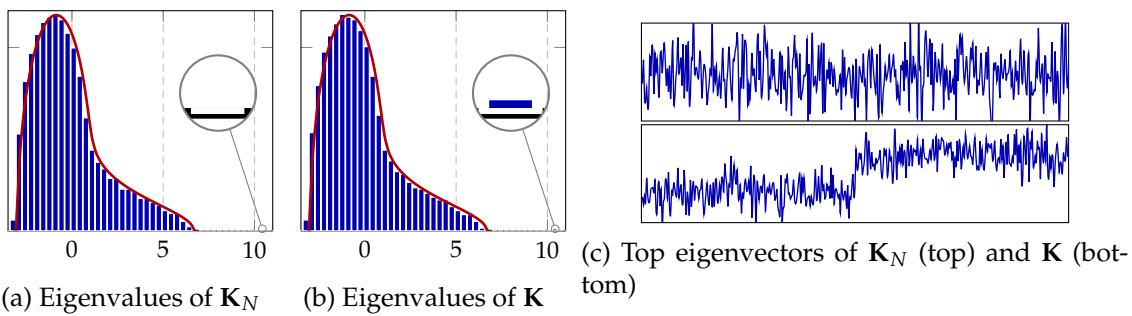
\begin{figure}[htb]
\centering
    \begin{subfigure}[c]{0.25\textwidth}
        \centering
        \begin{tikzpicture}[font=\footnotesize,spy using outlines]
        \renewcommand{\axisdefaulttryminticks}{4} 
        \pgfplotsset{every major grid/.style={densely dashed}}       
        \tikzstyle{every axis y label}+=[yshift=-10pt] 
        \tikzstyle{every axis x label}+=[yshift=5pt]
        \pgfplotsset{every axis legend/.style={cells={anchor=west},fill=white,
        at={(0.98,1)}, anchor=north east, font=\footnotesize }}
        \begin{axis}[
        width=1.35\textwidth,
        height=1.2\textwidth,
        xmin=-3.5,
        ymin=0,
        xmax=11,
        ymax=0.24,
        ytick={0,0.2,0.4,0.6,0.8},
        yticklabels = {},
        bar width=2pt,
        grid=major,
        ymajorgrids=false,
        scaled ticks=true,
        xlabel={},
        ylabel={}
        ]
        \addplot+[ybar,mark=none,color=white,fill=BLUE,area legend] coordinates{
        (-3.568400,0.000000)(-3.235239,0.011725)(-2.902078,0.105523)(-2.568916,0.162682)(-2.235755,0.193459)(-1.902594,0.213978)(-1.569433,0.227168)(-1.236272,0.233031)(-0.903111,0.235962)(-0.569950,0.230099)(-0.236789,0.215443)(0.096372,0.199322)(0.429533,0.168544)(0.762694,0.123110)(1.095855,0.087936)(1.429017,0.074746)(1.762178,0.067418)(2.095339,0.057158)(2.428500,0.057158)(2.761661,0.045434)(3.094822,0.045434)(3.427983,0.042502)(3.761144,0.035174)(4.094305,0.035174)(4.427466,0.030778)(4.760627,0.027846)(5.093788,0.021984)(5.426950,0.019053)(5.760111,0.019053)(6.093272,0.010259)(6.426433,0.004397)(6.759594,0.000000)(7.092755,0.000000)(7.425916,0.000000)(7.759077,0.000000)(8.092238,0.000000)(8.425399,0.000000)(8.758560,0.000000)(9.091721,0.000000)(9.424883,0.000000)(9.758044,0.000000)(10.091205,0.000000)(10.424366,0.000000)(10.757527,0.000000)(11.090688,0.000000)(11.423849,0.000000)(11.757010,0.000000)(12.090171,0.000000)(12.423332,0.000000)(12.756493,0.000000)
        };
        %\addlegendentry{{Eigenvalues of \(\K_N\)}};
        \addplot[smooth,RED,line width=1pt] plot coordinates{
        (-3.066958,0.000003)(-2.848848,0.082674)(-2.630738,0.129808)(-2.412628,0.160374)(-2.194517,0.182844)(-1.976407,0.199984)(-1.758297,0.213049)(-1.540187,0.222790)(-1.322077,0.229590)(-1.103967,0.233710)(-0.885857,0.235288)(-0.667747,0.234330)(-0.449637,0.230842)(-0.231527,0.224631)(-0.013417,0.215437)(0.204694,0.202819)(0.422804,0.186004)(0.640914,0.163717)(0.859024,0.134390)(1.077134,0.103928)(1.295244,0.086756)(1.513354,0.077201)(1.731464,0.070533)(1.949574,0.065270)(2.167684,0.060846)(2.385794,0.056994)(2.603905,0.053527)(2.822015,0.050354)(3.040125,0.047417)(3.258235,0.044694)(3.476345,0.042111)(3.694455,0.039636)(3.912565,0.037251)(4.130675,0.034946)(4.348785,0.032711)(4.566895,0.030511)(4.785005,0.028332)(5.003116,0.026153)(5.221226,0.023958)(5.439336,0.021754)(5.657446,0.019400)(5.875556,0.016975)(6.093666,0.014244)(6.311776,0.011139)(6.529886,0.007062)(6.747996,0.000000)(6.966106,0.000000)
        };
        \coordinate (spike) at (10.42,0.002);
        \coordinate (spike_spy) at (8,0.15);
        \end{axis}
        \begin{scope}
          \spy[black!50!white,size=1cm,circle,connect spies,magnification=10] on (spike) in node [fill=none] at (spike_spy);
        \end{scope}
        \end{tikzpicture}
        \caption{Eigenvalues of \(\K_N\)}
    \label{fig:eigs-K-N}
    \end{subfigure}%
    \hfill{}%
    \begin{subfigure}[c]{0.25\textwidth}
        \centering
        \begin{tikzpicture}[font=\footnotesize,spy using outlines]
        \renewcommand{\axisdefaulttryminticks}{4} 
        \pgfplotsset{every major grid/.style={densely dashed}}       
        \tikzstyle{every axis y label}+=[yshift=-10pt] 
        \tikzstyle{every axis x label}+=[yshift=5pt]
        \pgfplotsset{every axis legend/.style={cells={anchor=west},fill=white,
        at={(0.98,1)}, anchor=north east, font=\footnotesize }}
        \begin{axis}[
        width=1.35\textwidth,
        height=1.2\textwidth,
        xmin=-3.5,
        ymin=0,
        xmax=11,
        ymax=0.24,
        ytick={0,0.2,0.4,0.6,0.8},
        yticklabels = {},
        bar width=2pt,
        grid=major,
        ymajorgrids=false,
        scaled ticks=true,
        xlabel={},
        ylabel={}
        ]
        \addplot+[ybar,mark=none,color=white,fill=BLUE,area legend] coordinates{
        (-3.568400,0.000000)(-3.235239,0.011725)(-2.902078,0.106989)(-2.568916,0.158285)(-2.235755,0.194925)(-1.902594,0.212512)(-1.569433,0.225703)(-1.236272,0.235962)(-0.903111,0.233031)(-0.569950,0.230099)(-0.236789,0.218375)(0.096372,0.200787)(0.429533,0.168544)(0.762694,0.126042)(1.095855,0.087936)(1.429017,0.076211)(1.762178,0.064486)(2.095339,0.061555)(2.428500,0.052762)(2.761661,0.046899)(3.094822,0.045434)(3.427983,0.041037)(3.761144,0.036640)(4.094305,0.033709)(4.427466,0.033709)(4.760627,0.026381)(5.093788,0.021984)(5.426950,0.017587)(5.760111,0.017587)(6.093272,0.010259)(6.426433,0.002931)(6.759594,0.000000)(7.092755,0.000000)(7.425916,0.000000)(7.759077,0.000000)(8.092238,0.000000)(8.425399,0.000000)(8.758560,0.000000)(9.091721,0.000000)(9.424883,0.000000)(9.758044,0.000000)(10.091205,0.000000)(10.424366,0.002)(10.757527,0.000000)(11.090688,0.000000)(11.423849,0.000000)(11.757010,0.000000)(12.090171,0.000000)(12.423332,0.000000)(12.756493,0.000000)
        };
        %\addlegendentry{{Eigenvalues of \(\K\)}};
        \addplot[smooth,RED,line width=1pt] plot coordinates{
        (-3.066958,0.000003)(-2.848848,0.082674)(-2.630738,0.129808)(-2.412628,0.160374)(-2.194517,0.182844)(-1.976407,0.199984)(-1.758297,0.213049)(-1.540187,0.222790)(-1.322077,0.229590)(-1.103967,0.233710)(-0.885857,0.235288)(-0.667747,0.234330)(-0.449637,0.230842)(-0.231527,0.224631)(-0.013417,0.215437)(0.204694,0.202819)(0.422804,0.186004)(0.640914,0.163717)(0.859024,0.134390)(1.077134,0.103928)(1.295244,0.086756)(1.513354,0.077201)(1.731464,0.070533)(1.949574,0.065270)(2.167684,0.060846)(2.385794,0.056994)(2.603905,0.053527)(2.822015,0.050354)(3.040125,0.047417)(3.258235,0.044694)(3.476345,0.042111)(3.694455,0.039636)(3.912565,0.037251)(4.130675,0.034946)(4.348785,0.032711)(4.566895,0.030511)(4.785005,0.028332)(5.003116,0.026153)(5.221226,0.023958)(5.439336,0.021754)(5.657446,0.019400)(5.875556,0.016975)(6.093666,0.014244)(6.311776,0.011139)(6.529886,0.007062)(6.747996,0.000000)(6.966106,0.000000)
        };
        \coordinate (spike) at (10.42,0.002);
        \coordinate (spike_spy) at (8,0.15);
        \end{axis}
        \begin{scope}
          \spy[black!50!white,size=1cm,circle,connect spies,magnification=10] on (spike) in node [fill=none] at (spike_spy);
        \end{scope}
        \end{tikzpicture}
        \caption{Eigenvalues of \(\K\)}
      \label{fig:eigs-K}
    \end{subfigure}
    \hfill{}%
    \begin{subfigure}[c]{0.48\textwidth}
    \centering
        \begin{tikzpicture}[font=\footnotesize,spy using outlines]
        \renewcommand{\axisdefaulttryminticks}{4} 
        \pgfplotsset{every major grid/.style={densely dashed}}       
        \tikzstyle{every axis y label}+=[yshift=-10pt] 
        \tikzstyle{every axis x label}+=[yshift=5pt]
        \pgfplotsset{every axis legend/.style={cells={anchor=west},fill=white,
        at={(0.98,1)}, anchor=north east, font=\footnotesize}}
        \begin{axis}[
        width=1\textwidth,
        height=.4\textwidth,
        xmin=0,
        ymin=-0.05,
        xmax=410,
        ymax=0.05,
        xtick = \empty,
        ytick = \empty,
        grid=major,
        ymajorgrids=false,
        scaled ticks=true,
        xlabel={},
        ylabel={}
        ]
        \addplot[smooth,BLUE,line width=.5pt] coordinates{
        (1,-0.011715)(2,0.012531)(3,-0.027794)(4,-0.030240)(5,-0.015226)(6,-0.037121)(7,-0.030553)(8,0.016090)(9,-0.003953)(10,0.025625)(11,0.020208)(12,-0.022395)(13,0.019704)(14,-0.000256)(15,0.013523)(16,-0.007396)(17,-0.026061)(18,0.016325)(19,0.019531)(20,-0.001136)(21,0.019073)(22,0.029734)(23,0.008530)(24,-0.014253)(25,0.003471)(26,-0.020142)(27,0.006559)(28,-0.027555)(29,-0.000764)(30,0.032121)(31,-0.030922)(32,-0.026899)(33,-0.006038)(34,0.005579)(35,0.035274)(36,-0.000473)(37,0.005314)(38,0.014083)(39,-0.013787)(40,0.028475)(41,-0.026261)(42,-0.002527)(43,0.027732)(44,-0.019740)(45,0.031170)(46,-0.032430)(47,-0.024088)(48,0.012646)(49,0.014398)(50,0.062969)(51,0.000267)(52,0.003137)(53,0.004329)(54,-0.004865)(55,-0.005026)(56,0.020446)(57,0.011469)(58,-0.026837)(59,-0.008851)(60,0.006982)(61,0.004718)(62,-0.007441)(63,-0.033710)(64,-0.021509)(65,-0.012206)(66,0.002890)(67,-0.001203)(68,0.049037)(69,-0.018624)(70,0.002866)(71,-0.010115)(72,0.006862)(73,-0.023647)(74,-0.014324)(75,-0.003853)(76,-0.026882)(77,-0.004939)(78,0.024707)(79,-0.007674)(80,-0.009139)(81,-0.036444)(82,-0.010573)(83,0.003541)(84,0.025876)(85,0.003965)(86,-0.018060)(87,-0.027360)(88,0.010413)(89,0.011453)(90,0.021887)(91,0.026809)(92,0.040197)(93,0.007157)(94,-0.055454)(95,-0.007238)(96,-0.014252)(97,0.038808)(98,-0.017097)(99,0.028391)(100,0.031238)(101,0.017671)(102,-0.034370)(103,-0.011936)(104,-0.028498)(105,-0.010952)(106,0.019387)(107,-0.002971)(108,0.002943)(109,-0.012868)(110,-0.018102)(111,0.027281)(112,0.037245)(113,-0.008018)(114,0.021569)(115,-0.023933)(116,-0.020170)(117,0.004925)(118,0.046657)(119,0.031454)(120,0.017303)(121,0.026365)(122,-0.028121)(123,-0.006409)(124,-0.016269)(125,0.023254)(126,-0.003808)(127,0.013186)(128,-0.007955)(129,-0.019387)(130,-0.018305)(131,-0.022077)(132,-0.000660)(133,0.019361)(134,-0.032222)(135,0.016120)(136,-0.006875)(137,-0.017426)(138,0.000437)(139,-0.038481)(140,-0.019999)(141,0.025901)(142,-0.004424)(143,-0.006416)(144,-0.044412)(145,0.008661)(146,0.015671)(147,0.037819)(148,0.004807)(149,0.018446)(150,-0.000143)(151,-0.026202)(152,0.018780)(153,-0.026390)(154,-0.022479)(155,-0.016228)(156,-0.003798)(157,-0.003769)(158,0.001358)(159,0.032340)(160,-0.051862)(161,0.008178)(162,-0.043018)(163,0.002124)(164,0.040079)(165,-0.009662)(166,-0.011518)(167,-0.004586)(168,-0.019682)(169,-0.005950)(170,0.027427)(171,-0.008948)(172,0.016257)(173,0.026335)(174,-0.042267)(175,0.029226)(176,-0.039527)(177,-0.018085)(178,0.001489)(179,-0.006369)(180,0.010262)(181,-0.012423)(182,0.009232)(183,0.001786)(184,-0.015002)(185,0.002656)(186,-0.008928)(187,-0.023221)(188,-0.018259)(189,-0.019707)(190,-0.009501)(191,-0.003750)(192,-0.038673)(193,0.022828)(194,-0.015598)(195,0.015079)(196,-0.022833)(197,-0.008628)(198,-0.002358)(199,0.004420)(200,0.000489)(201,-0.019458)(202,-0.025238)(203,0.006395)(204,-0.012608)(205,-0.052210)(206,0.013402)(207,-0.022485)(208,-0.034535)(209,0.003942)(210,0.000450)(211,-0.005915)(212,-0.009412)(213,0.018776)(214,-0.013030)(215,-0.009979)(216,-0.049411)(217,0.003371)(218,0.018449)(219,0.020155)(220,-0.016505)(221,0.022274)(222,0.013154)(223,-0.003225)(224,-0.013397)(225,-0.020955)(226,0.018021)(227,-0.035373)(228,-0.043181)(229,0.028185)(230,-0.021984)(231,-0.041526)(232,-0.020866)(233,-0.010690)(234,0.006533)(235,-0.007776)(236,-0.003506)(237,0.003765)(238,0.017719)(239,-0.009984)(240,0.002649)(241,0.007090)(242,0.017665)(243,-0.004693)(244,0.001756)(245,-0.027000)(246,-0.013154)(247,-0.026428)(248,-0.031995)(249,0.005186)(250,0.003093)(251,0.009872)(252,-0.019265)(253,0.030664)(254,-0.004257)(255,0.012209)(256,-0.007596)(257,-0.038286)(258,0.050828)(259,-0.002876)(260,-0.009161)(261,0.016545)(262,-0.004954)(263,0.006202)(264,-0.008541)(265,-0.045649)(266,0.058688)(267,0.004840)(268,0.003162)(269,-0.057833)(270,-0.017085)(271,-0.010695)(272,0.030057)(273,0.041790)(274,0.019473)(275,-0.005049)(276,0.018542)(277,0.032443)(278,-0.012938)(279,0.029452)(280,-0.039555)(281,-0.006468)(282,0.020495)(283,-0.003898)(284,-0.009396)(285,0.012368)(286,0.005583)(287,-0.008755)(288,0.005478)(289,0.027360)(290,-0.038149)(291,0.002054)(292,-0.014340)(293,0.002171)(294,-0.013092)(295,-0.028746)(296,0.001423)(297,-0.019940)(298,-0.016542)(299,-0.007645)(300,-0.022217)(301,0.013378)(302,-0.048955)(303,0.023162)(304,-0.008868)(305,0.024395)(306,0.038562)(307,0.017759)(308,0.035415)(309,0.004390)(310,-0.010585)(311,0.000462)(312,0.015951)(313,-0.001950)(314,-0.023242)(315,-0.002266)(316,0.023946)(317,0.013529)(318,0.029976)(319,-0.002477)(320,-0.000897)(321,0.026515)(322,-0.020802)(323,-0.008562)(324,-0.013248)(325,0.001951)(326,0.019512)(327,-0.007248)(328,-0.023327)(329,-0.014106)(330,0.003149)(331,0.020527)(332,-0.007463)(333,0.025948)(334,-0.024614)(335,-0.039665)(336,-0.057622)(337,-0.047727)(338,0.006719)(339,0.003572)(340,0.009234)(341,-0.009145)(342,0.005749)(343,0.000662)(344,-0.026630)(345,0.017996)(346,0.000553)(347,-0.006401)(348,0.003923)(349,0.012576)(350,-0.007983)(351,-0.002362)(352,-0.000997)(353,0.005121)(354,0.021890)(355,-0.015039)(356,0.000107)(357,-0.027485)(358,-0.016829)(359,0.011405)(360,-0.007731)(361,0.000496)(362,-0.023229)(363,0.012753)(364,0.033128)(365,-0.005315)(366,-0.009822)(367,0.022006)(368,0.002224)(369,-0.016445)(370,-0.021965)(371,0.014422)(372,0.014033)(373,-0.056615)(374,-0.011001)(375,-0.022879)(376,-0.022537)(377,-0.017803)(378,-0.017631)(379,0.009657)(380,-0.004528)(381,0.045093)(382,0.002799)(383,-0.010216)(384,0.001318)(385,0.006346)(386,-0.042403)(387,0.011955)(388,0.013178)(389,-0.021439)(390,-0.002821)(391,-0.006260)(392,-0.005499)(393,-0.007782)(394,0.030482)(395,-0.023903)(396,-0.027228)(397,0.021028)(398,-0.031982)(399,-0.000668)(400,-0.008137)(401,-0.056514)(402,-0.021760)(403,0.031758)(404,0.046521)(405,-0.020261)(406,0.002335)(407,0.059908)(408,-0.001517)(409,-0.015916)(410,0.010995)
        };
        \end{axis}
        \end{tikzpicture}
    \vfill{}%
        \begin{tikzpicture}[font=\footnotesize,spy using outlines]
        \renewcommand{\axisdefaulttryminticks}{4} 
        \pgfplotsset{every major grid/.style={densely dashed}}       
        \tikzstyle{every axis y label}+=[yshift=-10pt] 
        \tikzstyle{every axis x label}+=[yshift=5pt]
        \pgfplotsset{every axis legend/.style={cells={anchor=west},fill=white,
        at={(0.98,1)}, anchor=north east, font=\footnotesize}}
        \begin{axis}[
        width=1\textwidth,
        height=.4\textwidth,
        xmin=0,
        ymin=-0.05,
        xmax=410,
        ymax=0.05,
        ytick={-0.05, 0, 0.05},
        xtick = \empty,
        ytick = \empty,
        grid=major,
        ymajorgrids=false,
        scaled ticks=true,
        xlabel={},
        ylabel={}
        ]
        \addplot[smooth,BLUE,line width=.5pt] coordinates{
        (1,-0.023638)(2,-0.031984)(3,-0.022181)(4,-0.030033)(5,-0.004275)(6,-0.021977)(7,-0.029416)(8,-0.027339)(9,-0.016521)(10,-0.014694)(11,-0.014565)(12,-0.019319)(13,-0.005972)(14,-0.037402)(15,-0.039795)(16,-0.012609)(17,-0.009725)(18,-0.033062)(19,-0.015511)(20,-0.055743)(21,-0.030054)(22,-0.025133)(23,-0.006374)(24,-0.008458)(25,-0.025599)(26,0.001589)(27,-0.035379)(28,-0.042342)(29,-0.017546)(30,-0.021719)(31,-0.009429)(32,-0.014639)(33,-0.021904)(34,-0.021503)(35,0.005928)(36,-0.024085)(37,-0.017644)(38,-0.011611)(39,-0.029971)(40,-0.029029)(41,0.003706)(42,-0.018345)(43,-0.042147)(44,-0.019248)(45,-0.027750)(46,-0.042358)(47,-0.006389)(48,-0.033909)(49,-0.015170)(50,-0.008413)(51,-0.027270)(52,-0.013321)(53,-0.013603)(54,-0.015276)(55,-0.025779)(56,-0.017351)(57,-0.015547)(58,-0.005784)(59,-0.025327)(60,-0.041577)(61,-0.032210)(62,-0.019762)(63,-0.009955)(64,-0.025224)(65,-0.031653)(66,-0.011798)(67,-0.020668)(68,-0.034331)(69,-0.009106)(70,-0.033746)(71,-0.007916)(72,-0.006970)(73,-0.029057)(74,-0.026248)(75,-0.011221)(76,-0.025376)(77,-0.005132)(78,-0.005895)(79,-0.018837)(80,0.009850)(81,-0.002948)(82,-0.011543)(83,-0.033404)(84,-0.011749)(85,-0.012302)(86,0.018077)(87,-0.020054)(88,-0.005482)(89,0.009018)(90,-0.004262)(91,-0.010487)(92,0.003933)(93,-0.025029)(94,-0.024642)(95,-0.010838)(96,-0.040499)(97,-0.020094)(98,-0.025087)(99,-0.010626)(100,-0.005969)(101,-0.017701)(102,-0.015499)(103,-0.007370)(104,-0.020520)(105,0.007947)(106,-0.044929)(107,-0.022903)(108,-0.001898)(109,-0.012967)(110,-0.053032)(111,-0.009326)(112,-0.014880)(113,-0.026857)(114,-0.014744)(115,-0.013125)(116,-0.019363)(117,-0.003956)(118,-0.004051)(119,-0.038141)(120,-0.021024)(121,-0.014619)(122,-0.019163)(123,-0.011047)(124,-0.020865)(125,-0.016889)(126,-0.010169)(127,-0.010535)(128,-0.026648)(129,-0.023921)(130,-0.016251)(131,-0.033863)(132,-0.039194)(133,-0.042004)(134,-0.001673)(135,-0.048972)(136,-0.035162)(137,-0.032648)(138,-0.006821)(139,-0.009101)(140,-0.000733)(141,-0.019585)(142,-0.020028)(143,-0.025948)(144,-0.022857)(145,-0.020852)(146,-0.006202)(147,-0.026585)(148,-0.030858)(149,-0.026582)(150,-0.016563)(151,-0.006086)(152,-0.000670)(153,-0.030223)(154,-0.004019)(155,-0.019043)(156,-0.004086)(157,-0.018079)(158,-0.032985)(159,-0.026705)(160,-0.023993)(161,-0.023187)(162,-0.033193)(163,-0.022953)(164,-0.029284)(165,-0.031144)(166,-0.016058)(167,0.004480)(168,-0.026223)(169,-0.021611)(170,-0.002775)(171,-0.024127)(172,-0.026721)(173,0.008759)(174,-0.016903)(175,-0.011480)(176,-0.025468)(177,-0.022468)(178,-0.010416)(179,-0.016760)(180,-0.013272)(181,-0.025042)(182,0.006078)(183,-0.024956)(184,-0.001244)(185,-0.031327)(186,-0.016983)(187,-0.015879)(188,-0.020241)(189,-0.024773)(190,-0.021978)(191,-0.040619)(192,-0.007289)(193,-0.021067)(194,-0.026644)(195,-0.018298)(196,-0.020268)(197,-0.035630)(198,-0.011853)(199,-0.045588)(200,-0.024422)(201,-0.026125)(202,-0.013029)(203,-0.000581)(204,-0.024167)(205,-0.013381)(206,0.025236)(207,0.004873)(208,0.006658)(209,0.020798)(210,0.015812)(211,0.019857)(212,-0.000612)(213,0.005407)(214,0.013365)(215,0.035109)(216,0.040373)(217,-0.001601)(218,0.032844)(219,0.020736)(220,0.004535)(221,0.022083)(222,0.008424)(223,0.009716)(224,0.003770)(225,0.015528)(226,0.010746)(227,0.013942)(228,0.015888)(229,0.023679)(230,0.004168)(231,-0.004594)(232,0.007384)(233,0.025717)(234,0.019123)(235,0.023000)(236,-0.002698)(237,0.004516)(238,0.024199)(239,0.013011)(240,0.015605)(241,-0.009503)(242,0.002564)(243,0.045318)(244,0.025936)(245,-0.006086)(246,0.022159)(247,0.012401)(248,0.026590)(249,0.010994)(250,0.030489)(251,0.009119)(252,0.022665)(253,0.020190)(254,-0.004280)(255,0.024017)(256,0.033715)(257,0.007714)(258,0.022160)(259,0.028375)(260,0.007638)(261,0.024759)(262,0.022408)(263,0.023641)(264,0.018793)(265,0.035382)(266,0.009903)(267,0.016624)(268,0.025422)(269,0.005735)(270,0.035081)(271,0.023018)(272,0.016056)(273,0.014880)(274,0.015077)(275,0.011003)(276,0.016930)(277,0.023992)(278,0.014667)(279,0.026733)(280,0.034578)(281,0.014382)(282,0.019658)(283,0.038417)(284,0.001915)(285,0.034446)(286,0.025692)(287,0.025581)(288,0.036329)(289,0.034312)(290,0.032703)(291,0.019636)(292,0.022448)(293,0.029788)(294,0.020024)(295,0.031004)(296,0.026791)(297,0.031818)(298,0.014479)(299,0.021754)(300,-0.004457)(301,0.016453)(302,0.007233)(303,0.033406)(304,0.012795)(305,0.033790)(306,0.027290)(307,0.004535)(308,0.018967)(309,0.028302)(310,0.016303)(311,0.040514)(312,0.028683)(313,0.015417)(314,0.020757)(315,0.028428)(316,0.018333)(317,0.016753)(318,0.009479)(319,0.025446)(320,0.036925)(321,0.019219)(322,0.004591)(323,0.007163)(324,0.004887)(325,0.033137)(326,0.009766)(327,0.024515)(328,0.008521)(329,0.004273)(330,0.019115)(331,0.036518)(332,0.022088)(333,0.030253)(334,0.028230)(335,0.026318)(336,0.030227)(337,0.013556)(338,0.018496)(339,0.033695)(340,0.011729)(341,0.003317)(342,0.013729)(343,0.016323)(344,0.012095)(345,0.035356)(346,0.014251)(347,0.028733)(348,-0.004425)(349,0.018973)(350,0.031684)(351,0.009387)(352,0.040105)(353,0.012639)(354,0.009955)(355,0.007931)(356,0.017389)(357,0.013054)(358,0.040082)(359,0.019031)(360,0.045513)(361,0.003029)(362,0.029579)(363,0.010900)(364,0.031816)(365,-0.010752)(366,-0.003258)(367,0.036859)(368,0.029059)(369,0.009251)(370,0.042733)(371,0.042062)(372,-0.008507)(373,0.030864)(374,0.025716)(375,0.043060)(376,0.008022)(377,0.009457)(378,0.018455)(379,0.011123)(380,0.032581)(381,0.025052)(382,0.025053)(383,-0.005526)(384,0.040563)(385,0.019343)(386,0.011535)(387,0.033451)(388,0.012417)(389,0.050854)(390,0.002457)(391,-0.004888)(392,0.017548)(393,0.005268)(394,0.018483)(395,0.014931)(396,-0.013557)(397,0.009085)(398,0.022604)(399,0.030217)(400,0.018360)(401,0.001328)(402,0.011730)(403,0.009271)(404,-0.000163)(405,0.020912)(406,0.014938)(407,0.014348)(408,0.037197)(409,0.018569)(410,0.020137)
        };
        \end{axis}
        \end{tikzpicture}
         \caption{Top eigenvectors of \(\K_N\) (top) and \(\K\) (bottom)}
       \label{fig:eigvecs-K-and-K_N}
    \end{subfigure}
    \caption{Eigenvalue distribution and top eigenvector of \(\K_N\) and \(\K\), together with the limiting spectral measure \(\mathcal L\) (from Theorem~\ref{theo:lsd}) in red; \(f(x) = \sign(x)\), Gaussian \(\z_i\), \(n=2\,048\), \(p=512\), \(\bmu_1 = -[3/2;\ \zo_{p-1}] = -\bmu_2\) and \(\E_1 = \E_2 = \zo\). \( \X = [\x_1, \ldots, \x_n] \in \RR^{p \times n} \) with \(\x_1, \ldots, \x_{n/2} \in \mathcal{C}_1\) and \(\x_{n/2+1}, \ldots, \x_n \in \mathcal{C}_2\). 
    }
\end{figure}

\section{Theoretical results}
\label{sec:main}

The main idea for the asymptotic analysis of \(\K\) comes in two steps: (i) first, by an expansion of \(\x_i^\T\x_j\) as a function of \(\z_i,\z_j\) and the statistical mixture model parameters \(\bmu, \E\), we decompose \(\x_i^\T\x_j\) (under Assumption~\ref{ass:growth-rate}) into successive orders of magnitudes with respect to \(p\); this, as we will show, further allows for a Taylor expansion of \(f(\x_i^\T\x_j/\sqrt{p})\) for at least twice differentiable functions \(f\) around its dominant term \(f(\z_i^\T\z_j/\sqrt{p})\). Then, (ii) we rely on the orthogonal polynomial approach of \cite{cheng2013spectrum} to ``linearize'' the resulting matrix terms \(\{f(\x_i^\T\x_j/\sqrt{p})\}\), \(\{f'(\x_i^\T\x_j/\sqrt{p})\}\) and \(\{f''(\x_i^\T\x_j/\sqrt{p})\}\) (all terms corresponding to higher order derivatives asymptotically vanish) and use Assumption~\ref{ass:poly} to extend the result to arbitrary square-summable \(f\).

\medskip

Our main conclusion is that \(\K\) asymptotically behaves like a matrix \(\tilde\K\) following a so-called ``spiked random matrix model'' in the sense that \(\tilde\K=\K_N+\tilde\K_I\) is the sum of the full-rank ``noise'' matrix \(\K_N\) having compact limiting spectrum (the support of \(\mathcal L\)) and a low-rank ``information'' matrix \(\tilde\K_I\) \cite{baik2005phase,benaych2011eigenvalues}.

%Before tackling the generic case of \(\z_i\) with independent bounded moment entries, we begin with the simpler Gaussian case. 

\subsection{Information-plus-noise decomposition of \texorpdfstring{\(\K\)}{K}}

We first show that \(\K\) can be asymptotically approximated as \(\K_N+\K_I\) with \(\K_N\) defined in \eqref{eq:def-K-n} and \(\K_I\) an additional (so far full-rank) term containing the statistical information of the mixture model.

\medskip

As announced, we start by decomposing \(\x_i^\T\x_j\) into a sequence of terms of successive orders of magnitude using Assumption~\ref{ass:growth-rate} and \(\x_i=\bmu_a+(\I_p+\E_a)^{\frac12}\z_i\), \(\x_j=\bmu_b+(\I_p+\E_b)^{\frac12}\z_j\) for \(\x_i\in\mathcal C_a\) and \(\x_j\in\mathcal C_b\). We have precisely, for \(i\neq j\),
%First note that, under the non-trivial classification condition in Assumption~\ref{ass:growth-rate}, one has, for \(\x_i \in \mathcal{C}_a, \x_j \in \mathcal{C}_b\), \(a,b \in \{1,2\}\) that
\begin{align}
  &\frac{ \x_i^\T \x_j }{\sqrt p } = \frac{ \bmu_a ^\T \bmu_b }{\sqrt p} + \frac1{\sqrt p} (\bmu_a^\T (\I_p + \E_b)^{\frac12 } \z_j + \bmu_b^\T (\I_p + \E_a)^{\frac12 } \z_i) + \frac1{\sqrt p} \z_i^\T (\I_p + \E_a)^{\frac12} (\I_p + \E_b)^{\frac12} \z_j \nonumber \\
  %%%
  &= \underbrace{\frac{ \z_i^\T \z_j }{\sqrt p} }_{O(1)} + \underbrace{\frac{ \z_i^\T (\E_a + \E_b) \z_j }{2\sqrt p} }_{ \equiv \A_{ij} = O(p^{-1/4})} + \underbrace{\frac{ \bmu_a ^\T \bmu_b + \bmu_a^\T \z_j + \bmu_b^\T \z_i }{\sqrt p} - \frac{ \z_i^\T (\E_a - \E_b)^2 \z_j }{8 \sqrt p} }_{\equiv \B_{ij} = O(p^{-1/2})} + o(p^{-1/2}) \label{eq:def-A-B}
\end{align}
where in particular we performed a Taylor expansion of \((\I_p + \E_a)^{\frac12}\) (since \(\|\E_a\|=O(p^{-\frac14})\)) around \(\I_p\), and used the fact that with high probability \(\z_i^\T \E_a\z_j = O(p^{1/4})\) and \(\z_i^\T (\E_a - \E_b)^2 \z_j = O(1)\). 

\medskip

As a consequence of this expansion, for at least twice differentiable \(f\in L^2(\mu_p)\), we have
\[
  \K_{ij} =\frac{f(\x_i^\T \x_j/\sqrt p)}{\sqrt p} = \frac{f(\z_i^\T \z_j/\sqrt{p})}{\sqrt p} + \frac{f'(\z_i^\T \z_j/\sqrt{p}) }{\sqrt p} (\A_{ij} + \B_{ij}) + \frac{f''(\z_i^\T \z_j/\sqrt{p}) }{2 \sqrt p}\A_{ij}^2+ o(p^{-1})
\]
where \(o(p^{-1})\) is understood with high probability and uniformly over \(i,j\in\{1,\ldots,n\}\). This \emph{entry-wise} expansion up to order \(o(p^{-1})\) is sufficient since, \emph{matrix-wise}, if \(\A_{ij}=o(p^{-1})\) uniformly on \(i,j\), from \(\| \A \| \leq p \| \A \|_\infty = p \max_{i,j} |\A_{ij}|\), we have \(\|\A\|=o(1)\) as \(n,p\to\infty\).

\medskip

In the particular case where \(f\) is a monomial of degree \(k\geq 2\), this implies the following result.
\begin{Proposition}[Monomial \(f\)]\label{prop:K-approx}
Under Assumptions~\ref{ass:growth-rate}--\ref{ass:poly}, let \(f(x) = x^k\), \(k \geq 2\). Then, as \(n,p \to \infty\),
\begin{equation}\label{eq:K-decomposition}
  \| \K - (\K_N + \K_I) \| \to 0
\end{equation}
almost surely, with \(\K_N\) defined in \eqref{eq:def-K-n} and
\begin{equation}\label{eq:def-K-I}
  %K_I =  \frac{k}{\sqrt{p}}(\sigma^2 Z^\T Z/\sqrt{p})^{\circ (k-1)} \circ (A+B) + \frac{k(k-1)}{2\sqrt{p}} (\sigma^2 Z^\T Z/\sqrt{p})^{\circ (k-2)} \circ (A)^{\circ 2}.
  \K_I =  \frac{k}{\sqrt{p}}( \Z^\T \Z/\sqrt{p})^{\circ (k-1)} \circ (\A+\B) + \frac{k(k-1)}{2\sqrt{p}} ( \Z^\T \Z/\sqrt{p})^{\circ (k-2)} \circ (\A)^{\circ 2} %- \diag(\cdot)
\end{equation}
for \(\Z = [\z_1, \ldots, \z_n] \in \RR^{p \times n}\) and \(\A,\B\in\RR^{n\times n}\) defined in \eqref{eq:def-A-B} with \(\A_{ii}=\B_{ii}=0\). Here \(\X\circ \Y\) denotes the Hadamard product between \(\X,\Y\) and \(\X^{\circ k}\) the \(k\)-th Hadamard power, i.e., \([\X^{\circ k}]_{ij} = (\X_{ij})^k\).
\end{Proposition}

Since \(f\in L^2(\mu)\) can be decomposed into its Hermite polynomials, Proposition~\ref{prop:K-approx} along with Theorem~\ref{theo:lsd} allows for an asymptotic quantification of \(\K\). However, the expression of \(\K_I\) in \eqref{eq:def-K-I} does not so far allow for a thorough understanding of the spectrum of \(\K\), due to (i) the delicate Hadamard products between purely random (\(\Z^\T \Z\)) and informative matrices (\(\A, \B\)) and (ii) the fact that \(\K_I\) is full rank (so that the resulting spectral properties of \(\K_N+\K_I\) remains intractable). We next show that, as \(n,p\to\infty\), \(\K_I\) admits a tractable low-rank approximation \(\tilde\K_I\), thereby leading to a spiked-model approximation of \(\K\). 

\subsection{Spiked-model approximation of \texorpdfstring{\(\K\)}{K}}

Let us then consider \(\K_I\) defined in \eqref{eq:def-K-I}, the \((i,j)\) entry of which can be written as the sum of terms containing \(\bmu_a, \bmu_b\) (treated separately) and random variables of the type
\[
  \phi = \frac{C}{\sqrt p} (\x^\T \y /\sqrt p)^\alpha (\x^\T \F \y)^\beta
\]
for independent random vectors \(\x,\y \in \RR^p\) with i.i.d.\@ zero mean, unit variance and finite moments (uniformly on \(p\)) entries, deterministic \(\F \in \RR^{p \times p}\), \(C\in\RR\), \(\alpha\in \mathbb N\) and \(\beta\in\{1,2\}\). %More precisely, according to Proposition~\ref{prop:K-I-approx}, \(\F\) may take the form of \(\frac12 (\E_a + \E_b)/ \sqrt p = O(p^{-3/4})\) or \(- \frac18 (\E_a - \E_b)^2/\sqrt p = O(p^{-1})\) for \(\beta = 1\) and \(\frac12 (\E_a + \E_b)/ \sqrt p\) in the case \(\beta = 2\).

For Gaussian \(\x,\y\), the expectation of \(\phi\) can be explicitly computed via an integral trick \cite{williams1997computing,louart2018random}. For more generic \(\x,\y\) with i.i.d.\@ bounded moment entries, a combinatorial argument controls the higher order moments of the expansion which asymptotically result in (matrix-wise) vanishing terms. See Sections~\ref{sec:sm-expectation-computation-Gaussian}--\ref{sec:sm-proof-prop-K-I-approx} of the supplementary material. This leads to the following result.

\begin{Proposition}[Low rank asymptotics of \(\K_I\)]\label{prop:K-I-approx}
Under Assumptions~\ref{ass:growth-rate}--\ref{ass:poly}, for \(f(x) = x^k\), \(k \ge 2\),
\[
  \| \K_I - \tilde\K_I \| \to 0
\]
almost surely as \(n,p \to \infty\), for \(\K_I\) defined in \eqref{eq:def-K-I} and
\begin{equation}\label{eq:def-bar-K-I}
  \tilde\K_I = \begin{cases} \frac{k!!}p (\J \M^\T \M \J^\T + \J \M^\T \Z + \Z^\T \M \J^\T), & \text{for \(k\) odd}\\ 
  %\frac{k(k-1)!!}p \J \left( \frac12 \mathbf{T} + \frac{3(k-2)}8 \mathbf{T}^2/\sqrt p+ \frac12 \S \right) \J^\T - \diag(\cdot), & \text{for \(k\) even}; \end{cases}
  \frac{k(k-1)!!}{2p} \J \left( \mathbf{T} + \S \right) \J^\T, & \text{for \(k\) even} \end{cases}
\end{equation}
where\footnote{For mental reminder, \(\M\) stands for \emph{means}, \(\mathbf{T}\) accounts for the difference in \emph{traces} of covariance matrices and \(\S\) for the ``shapes'' of the covariances.} 
\[
  \M = [\bmu_1,\bmu_2] \in \RR^{p \times 2}, \ \mathbf{T} = \{ \tr(\E_a + \E_b)/\sqrt p \}_{a,b=1}^2, \ \S = \{ \tr(\E_a \E_b)/ \sqrt p \}_{a,b=1}^2 \in \RR^{2 \times 2}
\]
and \(\J = [\j_1,\j_2] \in \RR^{n \times 2}\) with \(\j_a \in \RR^n\) the canonical vector of class \(\mathcal{C}_a\), i.e., \([\j_a]_i = \delta_{\x_i \in \mathcal{C}_a}\).
\end{Proposition}
We refer the readers to Section~\ref{sec:sm-proof-prop-K-I-approx} of the supplementary material for a detailed exposition of the proof.

%Note that according to Assumption~\ref{ass:growth-rate}, \(\mathbf{T}\) has in general \(O(p^{1/4})\) entries so that \(\mathbf{T}^2/\sqrt p\) is of order \(O(1)\). If 

\medskip

Proposition~\ref{prop:K-I-approx} states that \(\K_I\) is asymptotically equivalent to \(\tilde\K_I\) that is of rank at most two.\footnote{Note that, as defined, \(\tilde \K_I\) has non-zero diagonal elements, while \([\K_I]_{ii}=0\). This is not contradictory as the diagonal matrix \({\rm diag}(\tilde\K_I)\) has vanishing norm and can thus be added without altering the approximation \(\|\K_I-\tilde\K_I\|\to 0\); it however appears convenient as it ensures that \(\tilde\K_I\) is low rank (while without its diagonal, \(\tilde\K_I\) is full rank).} Note that the eigenvectors of \(\tilde\K_I\) are linear combinations of the vectors \(\j_1,\j_2\) and thus provide the data classes. 

From the expression of \(\tilde\K_I\), quite surprisingly, it appears that for \(f(x) = x^k\), depending on whether \(k\) is odd or even, either only the information in means (\(\M\) ) or only in covariance (\(\mathbf{T}\) and \(\S\)) can be (asymptotically) preserved.

\medskip

By merely combining the results of Propositions~\ref{prop:K-approx}--\ref{prop:K-I-approx}, the latter can be easily extended to polynomial \(f\). Then, by considering \(f(x)=P_\kappa(x)\), the Hermite polynomial of degree \(\kappa\), it can be shown that, quite surprisingly, one has \(\tilde \K_I= \zo\) if \(\kappa>2\) (see Section~\ref{sec:sm-proof-prop-K-approx-final} of the supplementary material). As such, using the Hermite polynomial expansion \(P_0,P_1,\ldots\) of an arbitrary \(f\in L^2(\mu)\) satisfying Assumption~\ref{ass:poly} leads to a very simple expression of our main result.
\begin{Theorem}[Spiked-model approximation of \(\K\)]\label{prop:K-approx-final}
For an arbitrary \(f\in L^2(\mu)\) with \(f\sim \sum_{l=0}^\infty a_l P_l(x)\), under Assumptions~\ref{ass:growth-rate}--\ref{ass:poly},
%for \(f(x) = P_\kappa(x)\) with \(P_\kappa(x)\) the Hermite polynomial of degree \(\kappa \ge 1\) and \(\z_i \sim \NN(\zo, \I_p)\), we have
\[
  \| \K - \tilde \K \| \to 0, \quad \tilde \K = \K_N + \tilde\K_I
\]
with \(\K_N\) defined in \eqref{eq:def-K-n} and
\begin{equation}
  %\tilde\K_I = \begin{cases} \frac1p (\J \M^\T \M \J^\T + \J \M^\T \Z + \Z^\T \M \J^\T ) - \diag(\cdot), &\text{for \(f(x) = P_1(x) = x\)}; \\ 
  %%%%
  %\frac1{p} \J (\mathbf{T} + \S) \J^\T - \diag(\cdot), &\text{for \(f(x) = P_2(x) = \frac1{\sqrt 2}(x^2 -1)\)}; \\
  %%%%
  %%\frac1{p} \J (9 \mathbf{T}^2/\sqrt p) \J^\T - \diag(\cdot), &\text{for \(\kappa =4\)}; %\\
  %%%%
  %\zo, & \text{otherwise}.
  %\end{cases}
  \tilde\K_I =  \frac{a_1}p (\J \M^\T \M \J^\T + \J \M^\T \Z + \Z^\T \M \J^\T ) + \frac{a_2}{p} \J (\mathbf{T} + \S) \J^\T.
\end{equation}
\end{Theorem}
%, the (normalized) Hermite polynomial of degree \(\kappa\) constructed as in Assumption~\ref{ass:poly}. Indeed, it is important to note that the (normalized) Hermite polynomial of degree \(k\) for \(k\) odd contains only \(x^{k!!}\), i.e., the monomials of odd degree, and similarly for \(k\) even. As such, to study \(f(x) = P_\kappa(x)\), it suffices to apply Proposition~\ref{prop:K-I-approx} and add up the associated coefficient of each monomial. This is formulated in the following proposition.
The detailed proof of Theorem~\ref{prop:K-approx-final} is provided in Section~\ref{sec:sm-proof-prop-K-approx-final} of the supplementary material.

%Theorem~\ref{prop:K-approx-final} unveils the surprising fact that, for Gaussian random vectors, although from Proposition~\ref{prop:K-I-approx} one knows that the monomial of degree \(k\), depending on \(k\) is odd or even, can always reserve the information in either means or covariances, which makes the clustering possible. It is no longer true if one consider the (normalized) Hermite polynomial function \(f(x) = P_\kappa(x)\) as long as \(\kappa \ge 3\). Indeed, due to the intrinsic relation between the coefficients \(c_{\kappa,l}\) of the Hermite polynomial and the moments of the Gaussian distribution. The information in means \(\M\) remains only for \(P_1(x)\) and the information in covariance \(\mathbf{T}\) and \(\S\) appear solely in \(P_2(x)\). For higher degree of Hermite polynomials, even though information terms appear in each monomial, the sum over all monomial equals to zero. 

\begin{figure}[htb]
\centering
    %%% empirical f1
    \begin{subfigure}[c]{0.25\textwidth}
        \centering
        \begin{tikzpicture}[font=\footnotesize,spy using outlines]
        \renewcommand{\axisdefaulttryminticks}{4} 
        \pgfplotsset{every major grid/.style={densely dashed}}       
        \tikzstyle{every axis y label}+=[yshift=-10pt] 
        \tikzstyle{every axis x label}+=[yshift=5pt]
        \pgfplotsset{every axis legend/.style={cells={anchor=west},fill=white,
        at={(0.98,1)}, anchor=north east, font=\footnotesize}}
        \begin{axis}[
        width=1.35\textwidth,
        height=1.2\textwidth,
        xmin=-1,
        ymin=0,
        xmax=2,
        ymax=0.9,
        xtick={-1,0,1,1.5},
        yticklabels = {},
        bar width=1.5pt,
        grid=major,
        ymajorgrids=false,
        scaled ticks=true,
        xlabel={},
        ylabel={}
        ]
        \addplot+[ybar,mark=none,color=white,fill=BLUE,area legend] coordinates{
        (-0.932564,0.000000)(-0.876357,0.000000)(-0.820150,0.052123)(-0.763942,0.469107)(-0.707735,0.642850)(-0.651528,0.729721)(-0.595321,0.799219)(-0.539113,0.816593)(-0.482906,0.799219)(-0.426699,0.764470)(-0.370492,0.816593)(-0.314284,0.781844)(-0.258077,0.747096)(-0.201870,0.712347)(-0.145663,0.712347)(-0.089456,0.642850)(-0.033248,0.625476)(0.022959,0.608101)(0.079166,0.573353)(0.135373,0.538604)(0.191581,0.521230)(0.247788,0.503855)(0.303995,0.503855)(0.360202,0.434358)(0.416410,0.416984)(0.472617,0.399609)(0.528824,0.399609)(0.585031,0.364861)(0.641238,0.330112)(0.697446,0.330112)(0.753653,0.295363)(0.809860,0.243240)(0.866067,0.277989)(0.922275,0.208492)(0.978482,0.208492)(1.034689,0.156369)(1.090896,0.173743)(1.147104,0.121620)(1.203311,0.034749)(1.259518,0.017374)(1.315725,0.000000)(1.371932,0.000000)(1.428140,0.000000)(1.484347,0.017374)(1.540554,0.000000)(1.596761,0.000000)(1.652969,0.000000)(1.709176,0.000000)(1.765383,0.000000)(1.821590,0.000000)
        };
        \addplot[smooth,RED,line width=.5pt] plot coordinates{
        (-0.932564,0.000011)(-0.876357,0.000015)(-0.820150,0.000027)(-0.763942,0.000096)(-0.707735,0.626535)(-0.651528,0.790466)(-0.595321,0.840424)(-0.539113,0.848426)(-0.482906,0.837548)(-0.426699,0.817554)(-0.370492,0.793042)(-0.314284,0.766450)(-0.258077,0.739061)(-0.201870,0.711542)(-0.145663,0.684330)(-0.089456,0.657643)(-0.033248,0.631524)(0.022959,0.606060)(0.079166,0.581245)(0.135373,0.557030)(0.191581,0.533336)(0.247788,0.510179)(0.303995,0.487498)(0.360202,0.465188)(0.416410,0.443181)(0.472617,0.421439)(0.528824,0.399909)(0.585031,0.378434)(0.641238,0.356967)(0.697446,0.335388)(0.753653,0.313603)(0.809860,0.291438)(0.866067,0.268706)(0.922275,0.245164)(0.978482,0.220434)(1.034689,0.193933)(1.090896,0.164749)(1.147104,0.130963)(1.203311,0.087208)(1.259518,0.000019)(1.315725,0.000007)(1.371932,0.000004)(1.428140,0.000003)(1.484347,0.000003)(1.540554,0.000002)(1.596761,0.000002)(1.652969,0.000002)(1.709176,0.000002)(1.765383,0.000001)(1.821590,0.000001)
        };
        \coordinate (spike) at (1.48,0.002);
        \coordinate (spike_spy) at (1.53,0.5);
        \end{axis}
        \begin{scope}
          \spy[black!50!white,size=1cm,circle,connect spies,magnification=5] on (spike) in node [fill=none] at (spike_spy);
        \end{scope}
        \end{tikzpicture}
        \caption{ \(\NN\) : eigs of \(\K\) for \(P_1\)}
    \label{fig:eigs-K-f1}
    \end{subfigure}%
    %%% No 2: approximation f1
    \begin{subfigure}[c]{0.25\textwidth}
        \centering
        \begin{tikzpicture}[font=\footnotesize,spy using outlines]
        \renewcommand{\axisdefaulttryminticks}{4} 
        \pgfplotsset{every major grid/.style={densely dashed}}       
        \tikzstyle{every axis y label}+=[yshift=-10pt] 
        \tikzstyle{every axis x label}+=[yshift=5pt]
        \pgfplotsset{every axis legend/.style={cells={anchor=west},fill=white,
        at={(0.98,1)}, anchor=north east, font=\footnotesize}}
        \begin{axis}[
        width=1.35\textwidth,
        height=1.2\textwidth,
        xmin=-1,
        ymin=0,
        xmax=2,
        ymax=0.9,
        xtick={-1,0,1,1.5},
        yticklabels = {},
        bar width=1.5pt,
        grid=major,
        ymajorgrids=false,
        scaled ticks=true,
        xlabel={},
        ylabel={}
        ]
        \addplot+[ybar,mark=none,color=white,fill=BLUE,area legend] coordinates{
        (-0.932564,0.000000)(-0.876357,0.000000)(-0.820150,0.000000)(-0.763942,0.382235)(-0.707735,0.677599)(-0.651528,0.833967)(-0.595321,0.833967)(-0.539113,0.868716)(-0.482906,0.816593)(-0.426699,0.799219)(-0.370492,0.799219)(-0.314284,0.747096)(-0.258077,0.729721)(-0.201870,0.677599)(-0.145663,0.694973)(-0.089456,0.625476)(-0.033248,0.625476)(0.022959,0.555978)(0.079166,0.608101)(0.135373,0.538604)(0.191581,0.503855)(0.247788,0.503855)(0.303995,0.469107)(0.360202,0.486481)(0.416410,0.399609)(0.472617,0.416984)(0.528824,0.382235)(0.585031,0.364861)(0.641238,0.364861)(0.697446,0.330112)(0.753653,0.277989)(0.809860,0.277989)(0.866067,0.260615)(0.922275,0.225866)(0.978482,0.208492)(1.034689,0.191118)(1.090896,0.121620)(1.147104,0.121620)(1.203311,0.052123)(1.259518,0.000000)(1.315725,0.000000)(1.371932,0.000000)(1.428140,0.000000)(1.484347,0.017374)(1.540554,0.000000)(1.596761,0.000000)(1.652969,0.000000)(1.709176,0.000000)(1.765383,0.000000)(1.821590,0.000000)
        };
        \addplot[smooth,RED,line width=.5pt] plot coordinates{
        (-0.932564,0.000011)(-0.876357,0.000015)(-0.820150,0.000027)(-0.763942,0.000096)(-0.707735,0.626535)(-0.651528,0.790466)(-0.595321,0.840424)(-0.539113,0.848426)(-0.482906,0.837548)(-0.426699,0.817554)(-0.370492,0.793042)(-0.314284,0.766450)(-0.258077,0.739061)(-0.201870,0.711542)(-0.145663,0.684330)(-0.089456,0.657643)(-0.033248,0.631524)(0.022959,0.606060)(0.079166,0.581245)(0.135373,0.557030)(0.191581,0.533336)(0.247788,0.510179)(0.303995,0.487498)(0.360202,0.465188)(0.416410,0.443181)(0.472617,0.421439)(0.528824,0.399909)(0.585031,0.378434)(0.641238,0.356967)(0.697446,0.335388)(0.753653,0.313603)(0.809860,0.291438)(0.866067,0.268706)(0.922275,0.245164)(0.978482,0.220434)(1.034689,0.193933)(1.090896,0.164749)(1.147104,0.130963)(1.203311,0.087208)(1.259518,0.000019)(1.315725,0.000007)(1.371932,0.000004)(1.428140,0.000003)(1.484347,0.000003)(1.540554,0.000002)(1.596761,0.000002)(1.652969,0.000002)(1.709176,0.000002)(1.765383,0.000001)(1.821590,0.000001)
        };
        \coordinate (spike) at (1.48,0.002);
        \coordinate (spike_spy) at (1.53,0.5);
        \end{axis}
        \begin{scope}
          \spy[black!50!white,size=1cm,circle,connect spies,magnification=5] on (spike) in node [fill=none] at (spike_spy);
        \end{scope}
        \end{tikzpicture}
        \caption{\(\NN\): eigs of \(\tilde \K\) for \(P_1\)}
    \label{fig:eigs-K-approx-f1}
    \end{subfigure}%
    %
    %
    %
    %
    %
    %
    %%% No 3: f_2 empirical
    \begin{subfigure}[c]{0.25\textwidth}
        \centering
        \begin{tikzpicture}[font=\footnotesize,spy using outlines]
        \renewcommand{\axisdefaulttryminticks}{4} 
        \pgfplotsset{every major grid/.style={densely dashed}}       
        \tikzstyle{every axis y label}+=[yshift=-10pt] 
        \tikzstyle{every axis x label}+=[yshift=5pt]
        \pgfplotsset{every axis legend/.style={cells={anchor=west},fill=white,
        at={(0.98,1)}, anchor=north east, font=\footnotesize}}
        \begin{axis}[
        width=1.35\textwidth,
        height=1.2\textwidth,
        xmin=-2.5,
        ymin=0,
        xmax=2.5,
        ymax=0.7,
        yticklabels = {},
        xtick={-2,-1,0,1,2},
        bar width=1.5pt,
        grid=major,
        ymajorgrids=false,
        scaled ticks=true,
        xlabel={},
        ylabel={}
        ]
        \addplot+[ybar,mark=none,color=white,fill=BLUE,area legend] coordinates{
        %(-2.921366,0.000000)(-2.787333,0.000000)(-2.653300,0.000000)(-2.519267,0.007286)(-2.385234,0.000000)(-2.251201,0.000000)(-2.117169,0.000000)(-1.983136,0.000000)(-1.849103,0.000000)(-1.715070,0.000000)(-1.581037,0.072860)(-1.447004,0.145720)(-1.312971,0.204008)(-1.178939,0.218580)(-1.044906,0.284154)(-0.910873,0.298726)(-0.776840,0.371586)(-0.642807,0.386158)(-0.508774,0.422588)(-0.374741,0.466304)(-0.240708,0.473590)(-0.106676,0.480876)(0.027357,0.488162)(0.161390,0.473590)(0.295423,0.437160)(0.429456,0.415302)(0.563489,0.378872)(0.697522,0.327870)(0.831554,0.306012)(0.965587,0.255010)(1.099620,0.204008)(1.233653,0.182150)(1.367686,0.116576)(1.501719,0.036430)(1.635752,0.000000)(1.769784,0.000000)(1.903817,0.000000)(2.037850,0.000000)(2.171883,0.000000)(2.305916,0.000000)(2.439949,0.000000)(2.573982,0.000000)(2.708014,0.000000)(2.842047,0.000000)(2.976080,0.007286)(3.110113,0.000000)(3.244146,0.000000)(3.378179,0.000000)(3.512212,0.000000)(3.646244,0.000000)
        (-2.081990,0.000000)(-1.987482,0.000000)(-1.892973,0.000000)(-1.798465,0.010333)(-1.703957,0.000000)(-1.609449,0.000000)(-1.514940,0.000000)(-1.420432,0.000000)(-1.325924,0.000000)(-1.231416,0.000000)(-1.136907,0.072332)(-1.042399,0.185996)(-0.947891,0.258327)(-0.853383,0.299660)(-0.758874,0.392658)(-0.664366,0.433990)(-0.569858,0.506322)(-0.475350,0.537321)(-0.380841,0.599319)(-0.286333,0.630319)(-0.191825,0.661318)(-0.097317,0.692317)(-0.002809,0.692317)(0.091700,0.661318)(0.186208,0.630319)(0.280716,0.609653)(0.375224,0.537321)(0.469733,0.485655)(0.564241,0.413324)(0.658749,0.392658)(0.753257,0.320326)(0.847766,0.278994)(0.942274,0.175663)(1.036782,0.092998)(1.131290,0.000000)(1.225799,0.000000)(1.320307,0.000000)(1.414815,0.000000)(1.509323,0.000000)(1.603831,0.000000)(1.698340,0.000000)(1.792848,0.000000)(1.887356,0.000000)(1.981864,0.000000)(2.076373,0.010333)(2.170881,0.000000)(2.265389,0.000000)(2.359897,0.000000)(2.454406,0.000000)(2.548914,0.000000)
        };
        \addplot[smooth,RED,line width=.5pt] plot coordinates{
        %(-2.921366,0.000000)(-2.787333,0.000001)(-2.653300,0.000001)(-2.519267,0.000001)(-2.385234,0.000001)(-2.251201,0.000001)(-2.117169,0.000001)(-1.983136,0.000001)(-1.849103,0.000002)(-1.715070,0.000002)(-1.581037,0.000004)(-1.447004,0.000012)(-1.312971,0.167220)(-1.178939,0.248600)(-1.044906,0.303327)(-0.910873,0.344329)(-0.776840,0.376139)(-0.642807,0.400962)(-0.508774,0.419998)(-0.374741,0.434048)(-0.240708,0.443579)(-0.106676,0.448862)(0.027357,0.450055)(0.161390,0.447226)(0.295423,0.440211)(0.429456,0.428881)(0.563489,0.412872)(0.697522,0.391603)(0.831554,0.364097)(0.965587,0.328877)(1.099620,0.283047)(1.233653,0.220060)(1.367686,0.114475)(1.501719,0.000006)(1.635752,0.000003)(1.769784,0.000002)(1.903817,0.000002)(2.037850,0.000001)(2.171883,0.000001)(2.305916,0.000001)(2.439949,0.000001)(2.573982,0.000001)(2.708014,0.000001)(2.842047,0.000000)(2.976080,0.000000)(3.110113,0.000000)(3.244146,0.000000)(3.378179,0.000000)(3.512212,0.000000)(3.646244,0.000000)
        (-2.081990,0.000001)(-1.987482,0.000001)(-1.892973,0.000001)(-1.798465,0.000001)(-1.703957,0.000001)(-1.609449,0.000002)(-1.514940,0.000002)(-1.420432,0.000003)(-1.325924,0.000003)(-1.231416,0.000005)(-1.136907,0.000007)(-1.042399,0.000016)(-0.947891,0.202769)(-0.853383,0.331831)(-0.758874,0.414562)(-0.664366,0.475800)(-0.569858,0.523113)(-0.475350,0.560075)(-0.380841,0.588637)(-0.286333,0.609944)(-0.191825,0.624807)(-0.097317,0.633576)(-0.002809,0.636621)(0.091700,0.633928)(0.186208,0.625488)(0.280716,0.610999)(0.375224,0.590113)(0.469733,0.561994)(0.564241,0.525574)(0.658749,0.478944)(0.753257,0.418693)(0.847766,0.337613)(0.942274,0.213124)(1.036782,0.000018)(1.131290,0.000007)(1.225799,0.000005)(1.320307,0.000003)(1.414815,0.000003)(1.509323,0.000002)(1.603831,0.000002)(1.698340,0.000002)(1.792848,0.000001)(1.887356,0.000001)(1.981864,0.000001)(2.076373,0.000001)(2.170881,0.000001)(2.265389,0.000001)(2.359897,0.000001)(2.454406,0.000001)(2.548914,0.000001)
        };
        \coordinate (spike1) at (2,0.01);
        \coordinate (spike1_spy) at (1.8,0.4);
        \coordinate (spike2) at (-1.8,0.01);
        \coordinate (spike2_spy) at (-1.8,0.4);
        \end{axis}
        \begin{scope}
          \spy[black!50!white,size=1cm,circle,connect spies,magnification=5] on (spike1) in node [fill=none] at (spike1_spy);
          \spy[black!50!white,size=1cm,circle,connect spies,magnification=5] on (spike2) in node [fill=none] at (spike2_spy);
        \end{scope}
        % \end{axis}
        % \begin{scope}
        %   \spy[black!50!white,size=1cm,circle,connect spies,magnification=3] on (2.85,0.05) in node [fill=none] at (2.65,1.5);
        %   \spy[black!50!white,size=1cm,circle,connect spies,magnification=3] on (0.35,0.05) in node [fill=none] at (0.45,1.5);
        % \end{scope}
        \end{tikzpicture}
        \caption{\(\NN\): eigs of \(\K\) for \(P_2\)}
      \label{fig:eigs-K-f2}
    \end{subfigure}%
    %%% No 4: f_2 approximation
    \begin{subfigure}[c]{0.25\textwidth}
        \centering
        \begin{tikzpicture}[font=\footnotesize,spy using outlines]
        \renewcommand{\axisdefaulttryminticks}{4} 
        \pgfplotsset{every major grid/.style={densely dashed}}       
        \tikzstyle{every axis y label}+=[yshift=-10pt] 
        \tikzstyle{every axis x label}+=[yshift=5pt]
        \pgfplotsset{every axis legend/.style={cells={anchor=west},fill=white,
        at={(0.98,1)}, anchor=north east, font=\footnotesize}}
        \begin{axis}[
        width=1.35\textwidth,
        height=1.2\textwidth,
        xmin=-2.5,
        ymin=0,
        xmax=2.5,
        ymax=0.7,
        yticklabels = {},
        xtick={-2,-1,0,1,2},
        bar width=1.5pt,
        grid=major,
        ymajorgrids=false,
        scaled ticks=true,
        xlabel={},
        ylabel={}
        ]
        \addplot+[ybar,mark=none,color=white,fill=BLUE,area legend] coordinates{
        %(-2.921366,0.000000)(-2.787333,0.000000)(-2.653300,0.007286)(-2.519267,0.000000)(-2.385234,0.000000)(-2.251201,0.000000)(-2.117169,0.000000)(-1.983136,0.000000)(-1.849103,0.000000)(-1.715070,0.000000)(-1.581037,0.000000)(-1.447004,0.087432)(-1.312971,0.196722)(-1.178939,0.269582)(-1.044906,0.327870)(-0.910873,0.349728)(-0.776840,0.400730)(-0.642807,0.415302)(-0.508774,0.422588)(-0.374741,0.451732)(-0.240708,0.459018)(-0.106676,0.437160)(0.027357,0.466304)(0.161390,0.444446)(0.295423,0.429874)(0.429456,0.422588)(0.563489,0.393444)(0.697522,0.378872)(0.831554,0.335156)(0.965587,0.298726)(1.099620,0.247724)(1.233653,0.174864)(1.367686,0.036430)(1.501719,0.000000)(1.635752,0.000000)(1.769784,0.000000)(1.903817,0.000000)(2.037850,0.000000)(2.171883,0.000000)(2.305916,0.000000)(2.439949,0.000000)(2.573982,0.000000)(2.708014,0.000000)(2.842047,0.007286)(2.976080,0.000000)(3.110113,0.000000)(3.244146,0.000000)(3.378179,0.000000)(3.512212,0.000000)(3.646244,0.000000)
        (-2.081990,0.000000)(-1.987482,0.000000)(-1.892973,0.010333)(-1.798465,0.000000)(-1.703957,0.000000)(-1.609449,0.000000)(-1.514940,0.000000)(-1.420432,0.000000)(-1.325924,0.000000)(-1.231416,0.000000)(-1.136907,0.000000)(-1.042399,0.082665)(-0.947891,0.237661)(-0.853383,0.371991)(-0.758874,0.444323)(-0.664366,0.506322)(-0.569858,0.547654)(-0.475350,0.588986)(-0.380841,0.588986)(-0.286333,0.619986)(-0.191825,0.650985)(-0.097317,0.630319)(-0.002809,0.661318)(0.091700,0.630319)(0.186208,0.609653)(0.280716,0.599319)(0.375224,0.578653)(0.469733,0.516655)(0.564241,0.506322)(0.658749,0.433990)(0.753257,0.392658)(0.847766,0.258327)(0.942274,0.103331)(1.036782,0.000000)(1.131290,0.000000)(1.225799,0.000000)(1.320307,0.000000)(1.414815,0.000000)(1.509323,0.000000)(1.603831,0.000000)(1.698340,0.000000)(1.792848,0.000000)(1.887356,0.000000)(1.981864,0.010333)(2.076373,0.000000)(2.170881,0.000000)(2.265389,0.000000)(2.359897,0.000000)(2.454406,0.000000)(2.548914,0.000000)
        };
        \addplot[smooth,RED,line width=.5pt] plot coordinates{
        %(-2.921366,0.000000)(-2.787333,0.000001)(-2.653300,0.000001)(-2.519267,0.000001)(-2.385234,0.000001)(-2.251201,0.000001)(-2.117169,0.000001)(-1.983136,0.000001)(-1.849103,0.000002)(-1.715070,0.000002)(-1.581037,0.000004)(-1.447004,0.000012)(-1.312971,0.167220)(-1.178939,0.248600)(-1.044906,0.303327)(-0.910873,0.344329)(-0.776840,0.376139)(-0.642807,0.400962)(-0.508774,0.419998)(-0.374741,0.434048)(-0.240708,0.443579)(-0.106676,0.448862)(0.027357,0.450055)(0.161390,0.447226)(0.295423,0.440211)(0.429456,0.428881)(0.563489,0.412872)(0.697522,0.391603)(0.831554,0.364097)(0.965587,0.328877)(1.099620,0.283047)(1.233653,0.220060)(1.367686,0.114475)(1.501719,0.000006)(1.635752,0.000003)(1.769784,0.000002)(1.903817,0.000002)(2.037850,0.000001)(2.171883,0.000001)(2.305916,0.000001)(2.439949,0.000001)(2.573982,0.000001)(2.708014,0.000001)(2.842047,0.000000)(2.976080,0.000000)(3.110113,0.000000)(3.244146,0.000000)(3.378179,0.000000)(3.512212,0.000000)(3.646244,0.000000)
        (-2.081990,0.000001)(-1.987482,0.000001)(-1.892973,0.000001)(-1.798465,0.000001)(-1.703957,0.000001)(-1.609449,0.000002)(-1.514940,0.000002)(-1.420432,0.000003)(-1.325924,0.000003)(-1.231416,0.000005)(-1.136907,0.000007)(-1.042399,0.000016)(-0.947891,0.202769)(-0.853383,0.331831)(-0.758874,0.414562)(-0.664366,0.475800)(-0.569858,0.523113)(-0.475350,0.560075)(-0.380841,0.588637)(-0.286333,0.609944)(-0.191825,0.624807)(-0.097317,0.633576)(-0.002809,0.636621)(0.091700,0.633928)(0.186208,0.625488)(0.280716,0.610999)(0.375224,0.590113)(0.469733,0.561994)(0.564241,0.525574)(0.658749,0.478944)(0.753257,0.418693)(0.847766,0.337613)(0.942274,0.213124)(1.036782,0.000018)(1.131290,0.000007)(1.225799,0.000005)(1.320307,0.000003)(1.414815,0.000003)(1.509323,0.000002)(1.603831,0.000002)(1.698340,0.000002)(1.792848,0.000001)(1.887356,0.000001)(1.981864,0.000001)(2.076373,0.000001)(2.170881,0.000001)(2.265389,0.000001)(2.359897,0.000001)(2.454406,0.000001)(2.548914,0.000001)
        };
        \coordinate (spike1) at (1.98,0.01);
        \coordinate (spike1_spy) at (1.8,0.4);
        \coordinate (spike2) at (-1.89,0.01);
        \coordinate (spike2_spy) at (-1.8,0.4);
        \end{axis}
        \begin{scope}
          \spy[black!50!white,size=1cm,circle,connect spies,magnification=5] on (spike1) in node [fill=none] at (spike1_spy);
          \spy[black!50!white,size=1cm,circle,connect spies,magnification=5] on (spike2) in node [fill=none] at (spike2_spy);
        \end{scope}
        % \end{axis}
        % \begin{scope}
        %   \spy[black!50!white,size=1cm,circle,connect spies,magnification=3] on (2.85,0.05) in node [fill=none] at (2.65,1.5);
        %   \spy[black!50!white,size=1cm,circle,connect spies,magnification=3] on (0.35,0.05) in node [fill=none] at (0.45,1.5);
        % \end{scope}
        \end{tikzpicture}
        \caption{\(\NN\): eigs of \(\tilde \K\) for \(P_2\)}
    \label{fig:eigs-K-approx-f2}
    \end{subfigure}%
\\
\medskip
\begin{subfigure}[c]{0.25\textwidth}
        \centering
        \begin{tikzpicture}[font=\footnotesize,spy using outlines]
        \renewcommand{\axisdefaulttryminticks}{4} 
        \pgfplotsset{every major grid/.style={densely dashed}}       
        \tikzstyle{every axis y label}+=[yshift=-10pt] 
        \tikzstyle{every axis x label}+=[yshift=5pt]
        \pgfplotsset{every axis legend/.style={cells={anchor=west},fill=white,
        at={(0.98,1)}, anchor=north east, font=\footnotesize}}
        \begin{axis}[
        width=1.35\textwidth,
        height=1.2\textwidth,
        xmin=-1,
        ymin=0,
        xmax=2,
        ymax=0.9,
        xtick={-1,0,1,1.5},
        yticklabels = {},
        bar width=1.5pt,
        grid=major,
        ymajorgrids=false,
        scaled ticks=true,
        xlabel={},
        ylabel={}
        ]
        \addplot+[ybar,mark=none,color=white,fill=BLUE,area legend] coordinates{
        (-0.937097,0.000000)(-0.879411,0.000000)(-0.821724,0.067715)(-0.764038,0.474006)(-0.706351,0.643294)(-0.648665,0.727938)(-0.590978,0.795653)(-0.533292,0.778724)(-0.475605,0.829510)(-0.417919,0.829510)(-0.360232,0.761795)(-0.302546,0.744867)(-0.244859,0.727938)(-0.187173,0.744867)(-0.129486,0.677151)(-0.071800,0.643294)(-0.014113,0.643294)(0.043573,0.558650)(0.101260,0.592507)(0.158946,0.541721)(0.216633,0.507864)(0.274319,0.474006)(0.332006,0.490935)(0.389692,0.440148)(0.447379,0.406291)(0.505065,0.423220)(0.562752,0.355504)(0.620438,0.338576)(0.678125,0.304718)(0.735811,0.304718)(0.793498,0.270861)(0.851184,0.253932)(0.908871,0.220074)(0.966557,0.203145)(1.024244,0.186217)(1.081930,0.152359)(1.139617,0.118501)(1.197303,0.084644)(1.254990,0.000000)(1.312676,0.000000)(1.370363,0.000000)(1.428049,0.000000)(1.485736,0.000000)(1.543422,0.016929)(1.601109,0.000000)(1.658795,0.000000)(1.716482,0.000000)(1.774168,0.000000)(1.831855,0.000000)(1.889541,0.000000)
        };
        \addplot[smooth,RED,line width=.5pt] plot coordinates{
        (-0.932564,0.000011)(-0.876357,0.000015)(-0.820150,0.000027)(-0.763942,0.000096)(-0.707735,0.626535)(-0.651528,0.790466)(-0.595321,0.840424)(-0.539113,0.848426)(-0.482906,0.837548)(-0.426699,0.817554)(-0.370492,0.793042)(-0.314284,0.766450)(-0.258077,0.739061)(-0.201870,0.711542)(-0.145663,0.684330)(-0.089456,0.657643)(-0.033248,0.631524)(0.022959,0.606060)(0.079166,0.581245)(0.135373,0.557030)(0.191581,0.533336)(0.247788,0.510179)(0.303995,0.487498)(0.360202,0.465188)(0.416410,0.443181)(0.472617,0.421439)(0.528824,0.399909)(0.585031,0.378434)(0.641238,0.356967)(0.697446,0.335388)(0.753653,0.313603)(0.809860,0.291438)(0.866067,0.268706)(0.922275,0.245164)(0.978482,0.220434)(1.034689,0.193933)(1.090896,0.164749)(1.147104,0.130963)(1.203311,0.087208)(1.259518,0.000019)(1.315725,0.000007)(1.371932,0.000004)(1.428140,0.000003)(1.484347,0.000003)(1.540554,0.000002)(1.596761,0.000002)(1.652969,0.000002)(1.709176,0.000002)(1.765383,0.000001)(1.821590,0.000001)
        };
        \coordinate (spike) at (1.51,0.002);
        \coordinate (spike_spy) at (1.53,0.5);
        \end{axis}
        \begin{scope}
          \spy[black!50!white,size=1cm,circle,connect spies,magnification=5] on (spike) in node [fill=none] at (spike_spy);
        \end{scope}
        % \end{axis}
        % \begin{scope}
        %   \spy[black!50!white,size=1cm,circle,connect spies,magnification=3] on (2.6,0.05) in node [fill=none] at (2.4,1.5);
        % \end{scope}
        \end{tikzpicture}
        \caption{Stud: eigs of \(\K\) for \(P_1\)}
    \label{fig:eigs-K-f1-student}
    \end{subfigure}%
    %%% No 2: approximation f1
    \begin{subfigure}[c]{0.25\textwidth}
        \centering
        \begin{tikzpicture}[font=\footnotesize,spy using outlines]
        \renewcommand{\axisdefaulttryminticks}{4} 
        \pgfplotsset{every major grid/.style={densely dashed}}       
        \tikzstyle{every axis y label}+=[yshift=-10pt] 
        \tikzstyle{every axis x label}+=[yshift=5pt]
        \pgfplotsset{every axis legend/.style={cells={anchor=west},fill=white,
        at={(0.98,1)}, anchor=north east, font=\footnotesize}}
        \begin{axis}[
        width=1.35\textwidth,
        height=1.2\textwidth,
        xmin=-1,
        ymin=0,
        xmax=2,
        ymax=0.9,
        xtick={-1,0,1,1.5},
        yticklabels = {},
        bar width=1.5pt,
        grid=major,
        ymajorgrids=false,
        scaled ticks=true,
        xlabel={},
        ylabel={}
        ]
        \addplot+[ybar,mark=none,color=white,fill=BLUE,area legend] coordinates{
        (-0.937097,0.000000)(-0.879411,0.000000)(-0.821724,0.000000)(-0.764038,0.355504)(-0.706351,0.761795)(-0.648665,0.778724)(-0.590978,0.846439)(-0.533292,0.846439)(-0.475605,0.829510)(-0.417919,0.812582)(-0.360232,0.761795)(-0.302546,0.727938)(-0.244859,0.694080)(-0.187173,0.727938)(-0.129486,0.677151)(-0.071800,0.626365)(-0.014113,0.626365)(0.043573,0.558650)(0.101260,0.592507)(0.158946,0.507864)(0.216633,0.507864)(0.274319,0.507864)(0.332006,0.474006)(0.389692,0.423220)(0.447379,0.423220)(0.505065,0.423220)(0.562752,0.355504)(0.620438,0.338576)(0.678125,0.338576)(0.735811,0.304718)(0.793498,0.253932)(0.851184,0.304718)(0.908871,0.220074)(0.966557,0.237003)(1.024244,0.135430)(1.081930,0.169288)(1.139617,0.118501)(1.197303,0.050786)(1.254990,0.000000)(1.312676,0.000000)(1.370363,0.000000)(1.428049,0.000000)(1.485736,0.000000)(1.543422,0.016929)(1.601109,0.000000)(1.658795,0.000000)(1.716482,0.000000)(1.774168,0.000000)(1.831855,0.000000)(1.889541,0.000000)
        };
        \addplot[smooth,RED,line width=.5pt] plot coordinates{
        (-0.932564,0.000011)(-0.876357,0.000015)(-0.820150,0.000027)(-0.763942,0.000096)(-0.707735,0.626535)(-0.651528,0.790466)(-0.595321,0.840424)(-0.539113,0.848426)(-0.482906,0.837548)(-0.426699,0.817554)(-0.370492,0.793042)(-0.314284,0.766450)(-0.258077,0.739061)(-0.201870,0.711542)(-0.145663,0.684330)(-0.089456,0.657643)(-0.033248,0.631524)(0.022959,0.606060)(0.079166,0.581245)(0.135373,0.557030)(0.191581,0.533336)(0.247788,0.510179)(0.303995,0.487498)(0.360202,0.465188)(0.416410,0.443181)(0.472617,0.421439)(0.528824,0.399909)(0.585031,0.378434)(0.641238,0.356967)(0.697446,0.335388)(0.753653,0.313603)(0.809860,0.291438)(0.866067,0.268706)(0.922275,0.245164)(0.978482,0.220434)(1.034689,0.193933)(1.090896,0.164749)(1.147104,0.130963)(1.203311,0.087208)(1.259518,0.000019)(1.315725,0.000007)(1.371932,0.000004)(1.428140,0.000003)(1.484347,0.000003)(1.540554,0.000002)(1.596761,0.000002)(1.652969,0.000002)(1.709176,0.000002)(1.765383,0.000001)(1.821590,0.000001)
        };
        \coordinate (spike) at (1.51,0.002);
        \coordinate (spike_spy) at (1.53,0.5);
        \end{axis}
        \begin{scope}
          \spy[black!50!white,size=1cm,circle,connect spies,magnification=5] on (spike) in node [fill=none] at (spike_spy);
        \end{scope}
        % \end{axis}
        % \begin{scope}
        %   \spy[black!50!white,size=1cm,circle,connect spies,magnification=3] on (2.6,0.05) in node [fill=none] at (2.4,1.5);
        % \end{scope}
        \end{tikzpicture}
        \caption{Stud: eigs of \(\tilde \K\) for \(P_1\)}
    \label{fig:eigs-K-approx-f1-student}
    \end{subfigure}%
    %
    %
    %
    %
    %
    %
    %%% No 3: f_2 empirical
    \begin{subfigure}[c]{0.25\textwidth}
        \centering
        \begin{tikzpicture}[font=\footnotesize,spy using outlines]
        \renewcommand{\axisdefaulttryminticks}{4} 
        \pgfplotsset{every major grid/.style={densely dashed}}       
        \tikzstyle{every axis y label}+=[yshift=-10pt] 
        \tikzstyle{every axis x label}+=[yshift=5pt]
        \pgfplotsset{every axis5 legend/.style={cells={anchor=west},fill=white,
        at={(0.98,1)}, anchor=north east, font=\footnotesize}}
        \begin{axis}[
        width=1.35\textwidth,
        height=1.2\textwidth,
        xmin=-2,
        ymin=0,
        xmax=2.5,
        ymax=0.75,
        yticklabels = {},
        xtick={-2,-1,0,1,2},
        bar width=1.2pt,
        grid=major,
        ymajorgrids=false,
        scaled ticks=true,
        xlabel={},
        ylabel={}
        ]
        \addplot+[ybar,mark=none,color=white,fill=BLUE,area legend] coordinates{
        %(-1.592784,0.000000)(-1.527795,0.000000)(-1.462807,0.000000)(-1.397819,0.000000)(-1.332831,0.045080)(-1.267843,0.105187)(-1.202855,0.105187)(-1.137866,0.165295)(-1.072878,0.150268)(-1.007890,0.195348)(-0.942902,0.240428)(-0.877914,0.240428)(-0.812926,0.270482)(-0.747938,0.315562)(-0.682949,0.375669)(-0.617961,0.405723)(-0.552973,0.420750)(-0.487985,0.510910)(-0.422997,0.525937)(-0.358009,0.571018)(-0.293020,0.631125)(-0.228032,0.661178)(-0.163044,0.661178)(-0.098056,0.721285)(-0.033068,0.691232)(0.031920,0.721285)(0.096908,0.706259)(0.161897,0.691232)(0.226885,0.631125)(0.291873,0.571018)(0.356861,0.555991)(0.421849,0.495884)(0.486837,0.405723)(0.551826,0.390696)(0.616814,0.360643)(0.681802,0.285509)(0.746790,0.300536)(0.811778,0.270482)(0.876766,0.210375)(0.941754,0.195348)(1.006743,0.180321)(1.071731,0.135241)(1.136719,0.120214)(1.201707,0.075134)(1.266695,0.075134)(1.331683,0.000000)(1.396672,0.000000)(1.461660,0.000000)(1.526648,0.000000)(1.591636,0.000000)
        (-1.475049,0.000000)(-1.414702,0.000000)(-1.354356,0.000000)(-1.294009,0.000000)(-1.233663,0.032365)(-1.173316,0.072822)(-1.112970,0.121370)(-1.052624,0.178009)(-0.992277,0.218465)(-0.931931,0.258922)(-0.871584,0.283196)(-0.811238,0.339835)(-0.750891,0.356017)(-0.690545,0.396474)(-0.630198,0.428839)(-0.569852,0.461204)(-0.509506,0.509752)(-0.449159,0.525935)(-0.388813,0.582574)(-0.328466,0.590665)(-0.268120,0.631121)(-0.207773,0.663487)(-0.147427,0.663487)(-0.087080,0.671578)(-0.026734,0.687761)(0.033613,0.663487)(0.093959,0.671578)(0.154305,0.639213)(0.214652,0.631121)(0.274998,0.606848)(0.335345,0.566391)(0.395691,0.534026)(0.456038,0.509752)(0.516384,0.461204)(0.576731,0.412656)(0.637077,0.388382)(0.697423,0.339835)(0.757770,0.315561)(0.818116,0.291287)(0.878463,0.258922)(0.938809,0.202283)(0.999156,0.169917)(1.059502,0.145643)(1.119849,0.064730)(1.180195,0.024274)(1.240541,0.000000)(1.300888,0.000000)(1.361234,0.000000)(1.421581,0.000000)(1.481927,0.000000)
        };
        \addplot[smooth,RED,line width=.5pt] plot coordinates{
        %(-1.592784,0.000002)(-1.527795,0.000002)(-1.462807,0.000002)(-1.397819,0.000003)(-1.332831,0.000003)(-1.267843,0.000004)(-1.202855,0.000005)(-1.137866,0.000007)(-1.072878,0.000011)(-1.007890,0.000044)(-0.942902,0.211986)(-0.877914,0.304801)(-0.812926,0.370729)(-0.747938,0.422551)(-0.682949,0.465006)(-0.617961,0.500491)(-0.552973,0.530406)(-0.487985,0.555651)(-0.422997,0.576859)(-0.358009,0.594413)(-0.293020,0.608662)(-0.228032,0.619834)(-0.163044,0.628111)(-0.098056,0.633547)(-0.033068,0.636249)(0.031920,0.636304)(0.096908,0.633602)(0.161897,0.628225)(0.226885,0.619996)(0.291873,0.608878)(0.356861,0.594711)(0.421849,0.577202)(0.486837,0.556064)(0.551826,0.530891)(0.616814,0.501065)(0.681802,0.465683)(0.746790,0.423363)(0.811778,0.371737)(0.876766,0.306127)(0.941754,0.214044)(1.006743,0.000048)(1.071731,0.000011)(1.136719,0.000007)(1.201707,0.000005)(1.266695,0.000004)(1.331683,0.000003)(1.396672,0.000003)(1.461660,0.000002)(1.526648,0.000002)(1.591636,0.000002)
        (-1.475049,0.000002)(-1.414702,0.000003)(-1.354356,0.000003)(-1.294009,0.000004)(-1.233663,0.000005)(-1.173316,0.000006)(-1.112970,0.000008)(-1.052624,0.000014)(-0.992277,0.079014)(-0.931931,0.230821)(-0.871584,0.312065)(-0.811238,0.372219)(-0.750891,0.420420)(-0.690545,0.460432)(-0.630198,0.494271)(-0.569852,0.523115)(-0.509506,0.547796)(-0.449159,0.568769)(-0.388813,0.586516)(-0.328466,0.601306)(-0.268120,0.613319)(-0.207773,0.622719)(-0.147427,0.629658)(-0.087080,0.634203)(-0.026734,0.636379)(0.033613,0.636257)(0.093959,0.633800)(0.154305,0.628979)(0.214652,0.621759)(0.274998,0.612066)(0.335345,0.599737)(0.395691,0.584638)(0.456038,0.566547)(0.516384,0.545168)(0.576731,0.520063)(0.637077,0.490680)(0.697423,0.456225)(0.757770,0.415379)(0.818116,0.366058)(0.878463,0.304145)(0.938809,0.219227)(0.999156,0.025774)(1.059502,0.000013)(1.119849,0.000008)(1.180195,0.000006)(1.240541,0.000004)(1.300888,0.000004)(1.361234,0.000003)(1.421581,0.000003)(1.481927,0.000002)
        };
        \end{axis}
        % \begin{scope}
        %   \spy[black!50!white,size=1cm,circle,connect spies,magnification=3] on (2.85,0.05) in node [fill=none] at (2.6,1.5);
        % \end{scope}
        \end{tikzpicture}
        \caption{Stud: eigs of \(\K\) for \(P_3\)}
      \label{fig:eigs-K-f3-student}
    \end{subfigure}%
    %%% No 4: f_2 approximation
    \begin{subfigure}[c]{0.25\textwidth}
        \centering
        \begin{tikzpicture}[font=\footnotesize,spy using outlines]
        \renewcommand{\axisdefaulttryminticks}{4} 
        \pgfplotsset{every major grid/.style={densely dashed}}       
        \tikzstyle{every axis y label}+=[yshift=-10pt] 
        \tikzstyle{every axis x label}+=[yshift=5pt]
        \pgfplotsset{every axis legend/.style={cells={anchor=west},fill=white,
        at={(0.98,1)}, anchor=north east, font=\footnotesize}}
        \begin{axis}[
        width=1.35\textwidth,
        height=1.2\textwidth,
        xmin=-2,
        ymin=0,
        xmax=2.5,
        ymax=0.7,
        yticklabels = {},
        xtick={-2,-1,0,1,2},
        bar width=1.2pt,
        grid=major,
        ymajorgrids=false,
        scaled ticks=true,
        xlabel={},
        ylabel={}
        ]
        \addplot+[ybar,mark=none,color=white,fill=BLUE,area legend] coordinates{
        %(-1.592784,0.000000)(-1.527795,0.000000)(-1.462807,0.000000)(-1.397819,0.000000)(-1.332831,0.000000)(-1.267843,0.000000)(-1.202855,0.000000)(-1.137866,0.000000)(-1.072878,0.030054)(-1.007890,0.150268)(-0.942902,0.270482)(-0.877914,0.300536)(-0.812926,0.360643)(-0.747938,0.420750)(-0.682949,0.450803)(-0.617961,0.480857)(-0.552973,0.540964)(-0.487985,0.571018)(-0.422997,0.586044)(-0.358009,0.601071)(-0.293020,0.616098)(-0.228032,0.661178)(-0.163044,0.646151)(-0.098056,0.676205)(-0.033068,0.676205)(0.031920,0.646151)(0.096908,0.631125)(0.161897,0.676205)(0.226885,0.646151)(0.291873,0.586044)(0.356861,0.571018)(0.421849,0.586044)(0.486837,0.510910)(0.551826,0.525937)(0.616814,0.450803)(0.681802,0.405723)(0.746790,0.375669)(0.811778,0.270482)(0.876766,0.255455)(0.941754,0.165295)(1.006743,0.045080)(1.071731,0.000000)(1.136719,0.000000)(1.201707,0.000000)(1.266695,0.000000)(1.331683,0.000000)(1.396672,0.000000)(1.461660,0.000000)(1.526648,0.000000)(1.591636,0.000000)
        (-1.475049,0.000000)(-1.414702,0.000000)(-1.354356,0.000000)(-1.294009,0.000000)(-1.233663,0.000000)(-1.173316,0.000000)(-1.112970,0.000000)(-1.052624,0.064730)(-0.992277,0.186100)(-0.931931,0.250830)(-0.871584,0.331743)(-0.811238,0.372200)(-0.750891,0.420748)(-0.690545,0.461204)(-0.630198,0.493569)(-0.569852,0.542117)(-0.509506,0.542117)(-0.449159,0.566391)(-0.388813,0.598756)(-0.328466,0.606848)(-0.268120,0.639213)(-0.207773,0.631121)(-0.147427,0.639213)(-0.087080,0.655395)(-0.026734,0.655395)(0.033613,0.631121)(0.093959,0.639213)(0.154305,0.647304)(0.214652,0.631121)(0.274998,0.614939)(0.335345,0.590665)(0.395691,0.566391)(0.456038,0.550208)(0.516384,0.525935)(0.576731,0.485478)(0.637077,0.445022)(0.697423,0.412656)(0.757770,0.364109)(0.818116,0.331743)(0.878463,0.234648)(0.938809,0.194191)(0.999156,0.048548)(1.059502,0.000000)(1.119849,0.000000)(1.180195,0.000000)(1.240541,0.000000)(1.300888,0.000000)(1.361234,0.000000)(1.421581,0.000000)(1.481927,0.000000)
        };
        \addplot[smooth,RED,line width=.5pt] plot coordinates{
        %(-1.592784,0.000002)(-1.527795,0.000002)(-1.462807,0.000002)(-1.397819,0.000003)(-1.332831,0.000003)(-1.267843,0.000004)(-1.202855,0.000005)(-1.137866,0.000007)(-1.072878,0.000011)(-1.007890,0.000044)(-0.942902,0.211986)(-0.877914,0.304801)(-0.812926,0.370729)(-0.747938,0.422551)(-0.682949,0.465006)(-0.617961,0.500491)(-0.552973,0.530406)(-0.487985,0.555651)(-0.422997,0.576859)(-0.358009,0.594413)(-0.293020,0.608662)(-0.228032,0.619834)(-0.163044,0.628111)(-0.098056,0.633547)(-0.033068,0.636249)(0.031920,0.636304)(0.096908,0.633602)(0.161897,0.628225)(0.226885,0.619996)(0.291873,0.608878)(0.356861,0.594711)(0.421849,0.577202)(0.486837,0.556064)(0.551826,0.530891)(0.616814,0.501065)(0.681802,0.465683)(0.746790,0.423363)(0.811778,0.371737)(0.876766,0.306127)(0.941754,0.214044)(1.006743,0.000048)(1.071731,0.000011)(1.136719,0.000007)(1.201707,0.000005)(1.266695,0.000004)(1.331683,0.000003)(1.396672,0.000003)(1.461660,0.000002)(1.526648,0.000002)(1.591636,0.000002)
        (-1.475049,0.000002)(-1.414702,0.000003)(-1.354356,0.000003)(-1.294009,0.000004)(-1.233663,0.000005)(-1.173316,0.000006)(-1.112970,0.000008)(-1.052624,0.000014)(-0.992277,0.079014)(-0.931931,0.230821)(-0.871584,0.312065)(-0.811238,0.372219)(-0.750891,0.420420)(-0.690545,0.460432)(-0.630198,0.494271)(-0.569852,0.523115)(-0.509506,0.547796)(-0.449159,0.568769)(-0.388813,0.586516)(-0.328466,0.601306)(-0.268120,0.613319)(-0.207773,0.622719)(-0.147427,0.629658)(-0.087080,0.634203)(-0.026734,0.636379)(0.033613,0.636257)(0.093959,0.633800)(0.154305,0.628979)(0.214652,0.621759)(0.274998,0.612066)(0.335345,0.599737)(0.395691,0.584638)(0.456038,0.566547)(0.516384,0.545168)(0.576731,0.520063)(0.637077,0.490680)(0.697423,0.456225)(0.757770,0.415379)(0.818116,0.366058)(0.878463,0.304145)(0.938809,0.219227)(0.999156,0.025774)(1.059502,0.000013)(1.119849,0.000008)(1.180195,0.000006)(1.240541,0.000004)(1.300888,0.000004)(1.361234,0.000003)(1.421581,0.000003)(1.481927,0.000002)
        };
        \end{axis}
        % \begin{scope}
        %   \spy[black!50!white,size=1cm,circle,connect spies,magnification=3] on (2.85,0.05) in node [fill=none] at (2.6,1.5);
        % \end{scope}
        \end{tikzpicture}
        \caption{Stud: eigs of \(\tilde \K\) for \(P_3\)}
    \label{fig:eigs-K-approx-f3-student}
    \end{subfigure}%
    \caption{Eigenvalue distributions of \(\K\) and \(\tilde\K\) from Theorem~\ref{prop:K-approx-final} (blue) and \(\mathcal L\) from Theorem~\ref{theo:lsd} (red), for \(\z_i\) with Gaussian (top) or Student-t with degree of freedom \(7\) (bottom) entries; functions \(f(x) = P_1(x) = x\), \(f(x) = P_2(x) =(x^2 - 1)\sqrt 2\), \(f(x) = P_3(x) = (x^3 - 3x)/\sqrt 6\); \(n=2\,048\), \(p=8\,192\), \(\bmu_1 = -[2;\zo_{p-1}] = -\bmu_2\) and \(\E_1 = - 10\I_p/\sqrt p = -\E_2\).}
    \label{fig:validation-approx}
\end{figure}
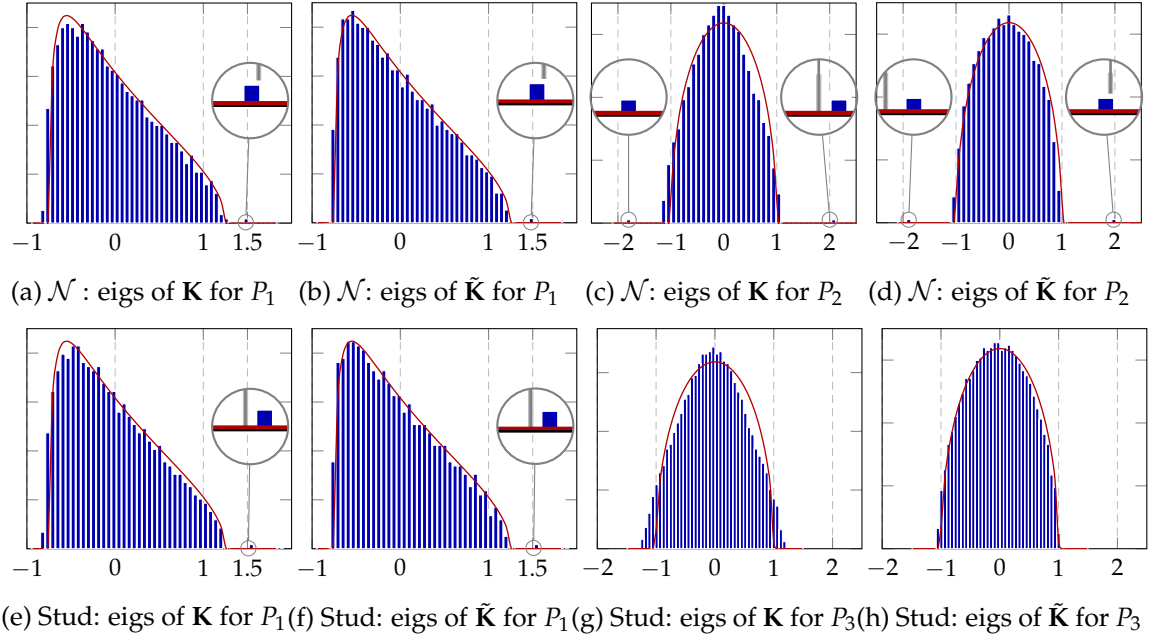

Figure~\ref{fig:validation-approx} compares the spectra of \(\K\) and \(\tilde\K\) for random vectors with independent Gaussian or Student-t entries, for the first three (normalized) Hermite polynomials \(P_1(x)\), \(P_2(x)\) and \(P_3(x)\). These numerical evidences validate Theorem~\ref{prop:K-approx-final}: only for \(P_1(x)\) and \(P_2(x)\) is an isolated eigenvalue observed. Besides, as shown in the bottom display of Figure~\ref{fig:eigvecs-K-and-K_N}, the corresponding eigenvector is, as expected, a noisy version of linear combinations of \(\j_1,\j_2\).

\begin{Remark}[Even and odd \(f\)]
While \({\rm rank}(\tilde \K_I)\leq 4\) (as the sum of two rank-two terms), in Figure~\ref{fig:validation-approx} \emph{no more than} two isolated eigenvalues are observed (for \(f = P_1\) only one on the right side, for \(f = P_2\) one on each side). This follows from \(a_2=0\) when \(f=P_1\) and \(a_1 = 0\) for \(f=P_2\). More generally, for \(f\) \emph{odd} (\(f(-x)=-f(x)\) ), \(a_2 = 0\) and the statistical information on covariances (through \(\E\)) asymptotically vanishes in \(\K\); for \(f\) \emph{even} (\(f(-x)=f(x)\)), \(a_1 = 0\) and information about the means \(\bmu_1,\bmu_2\) vanishes. Thus, only \(f\) neither odd nor even can preserve both first and second order discriminating statistics (e.g., the popular ReLU function \(f(x) = \max(0,x)\)). This was previously remarked in \cite{liao2018spectrum} based on a local expansion of smooth \(f\) in a similar setting.
\end{Remark}

\section{Practical consequences: universality of binary kernels}
\label{sec:practice}

As a direct consequence of Theorem~\ref{prop:K-approx-final}, the performance of spectral clustering for large dimensional mixture models of the type \eqref{eq:mixture} only depends on the \emph{three} parameters of the nonlinear function \(f\): \(a_1 = \EE[\xi f(\xi)]\), \(a_2 = \EE[\xi^2 f(\xi)]/\sqrt 2\) and \(\nu = \EE[f^2(\xi)]\). The parameters \(a_1,\nu\) determine the limiting spectral measure \(\mathcal L\) of \(\K\) (since \(\K\) and \(\K_N\) asymptotically differ by a rank-\(4\) matrix, they share the same limiting spectral measure) while \(a_2,a_2\) determine the low rank structure within \(\tilde\K_I\). 

As an immediate consequence, arbitrary (square-summable) kernel functions \(f\) (with \(a_0=0\)) are asymptotically \emph{equivalent} to the simple cubic function \(\tilde f(x) = c_3 x^3 + c_2 x^2 + c_1 x - c_2\) having the \emph{same} Hermite polynomial coefficients \(a_1, a_2 ,\nu\).\footnote{The coefficients being related through \(a_1 = 3 c_3 + c_1\), \(a_2 = \sqrt 2 c_2\) and \(\nu = (3c_3 + c_1)^2 + 6 c_3^2 + 2c_2^2\).}

\medskip

The idea of this section is to design a prototypical family \(\mathcal F\) of functions \(f\) having (i) universal properties with respect to \((a_1,a_2,\nu)\), i.e., for each \((a_1,a_2,\nu)\) there exists \(f\in\mathcal F\) with these Hermite coefficients and (ii) having numerically advantageous properties. Thus, any arbitrary kernel function \(f\) can be mapped, through \((a_1,a_2,\nu)\), to a function in \(\mathcal F\) with good numerical properties.

%Moreover, one can indeed construct much simpler functions \(f\) which share the same coefficients with a cubic function (thus of equivalent clustering performance) with a huge gain in terms of both allocated memory size and running time. For instance, consider the piecewise function

\medskip

One such prototypical family \(\mathcal F\) can be the set of \(f\), parametrized by \((t,s_-,s_+)\), and defined as
\begin{equation}\label{eq:def-piecewise-func}
    f(x) = \begin{cases} -r t \quad & x \le \sqrt 2 s_- \\ 0 \quad & \sqrt 2 s_- < x \le \sqrt 2 s_+ \\ t \quad & x > \sqrt 2 s_+ \end{cases}  \text{with} \quad \begin{cases} a_1 = \frac{t}{\sqrt{2\pi}} ( e^{-s_+^2} + r e^{-s_-^2} ) \\
    a_2 = \frac{t}{\sqrt{2\pi}} ( s_+ e^{-s_+^2} + r s_- e^{-s_-^2} ) \\
    \nu = \frac{t^2}2 \left(1-\erf(s_+) \right) ( 1 + r ) \end{cases}
\end{equation}
\begin{figure}[t!]
\begin{minipage}[c]{0.48\textwidth}%
\centering
%(\(100\) interations of power method)
  \captionof{table}{Storage size and top eigenvector running time of \(\K\) for piecewise constant and cubic \(f\), in the setting of Figure~\ref{fig:validation-approx} and~\ref{fig:piecewise-versus-cubic}.}
  \label{tab:size-time-f}
  \centering
  \begin{tabular}{lll}
    \toprule
    %\multicolumn{2}{c}{Part}                   \\
    %\cmidrule(r){1-2}
    \(f\)     & Size (Mb)     & Running time (s) \\
    \midrule
    Piecewise & \(4.15\)  & \(0.2390\)     \\
    Cubic     & \(16.75\) & \(0.4244\)      \\
    \bottomrule
  \end{tabular}
\end{minipage}
\hfill{}%
\begin{minipage}[c]{0.48\textwidth}%
\centering
\begin{tikzpicture}[scale=.6]
  \pgfplotsset{every axis legend/.style={cells={anchor=west},fill=white,
  at={(0.02,1)}, anchor=north west, font=\footnotesize}}
  \draw[->] (-3,0) -- (3,0) node[right] {};
  \draw[->] (0,-1) -- (0,3) node[above] {};
%   %
%   %
%   %
  \draw[scale=1.5,domain=-2:0,smooth,variable=\x,GREEN,line width=1.5pt] plot ({\x},{-0.314598});
  \draw[scale=1.5,domain=-0.314598:0,densely dashed,variable=\y,GREEN,line width=1.5pt] plot ({0},{\y});
  \draw[scale=1.5,domain=0:1.412826,smooth,variable=\x,GREEN,line width=1.5pt] plot ({\x},{0});
  \draw[scale=1.5,domain=0:2,densely dashed,variable=\y,GREEN,line width=1.5pt] plot ({1.412826},{\y});
  \draw[scale=1.5,domain=1.432866:2,smooth,variable=\x,GREEN,line width=1.5pt] plot ({\x},{2});
  \draw[scale=1.5,domain=-2:2,smooth,variable=\x,BLUE,line width=1pt] plot ({\x},{0.1306*\x*\x*\x+0.2076*\x*\x+0.0272*\x-0.2076});
  %
  %
  %
  %\draw[scale=2,domain=-2:2,smooth,variable=\x,GREEN,line width=1pt] plot ({\x},{tanh(\x)});
  %
  %
  %
%   \draw[scale=2,domain=-2:-0.4327,smooth,variable=\x,BLUE,line width=1.5pt] plot ({\x},{-0.8744});
%   \draw[scale=2,domain=-0.8744:0,densely dashed,variable=\y,BLUE,line width=1.5pt] plot ({-0.4327},{\y});
%   %
%   \draw[scale=2,domain=-0.4327:0.4327,smooth,variable=\x,BLUE,line width=1.5pt] plot ({\x},{0});
%   \draw[scale=2,domain=0:0.8744,densely dashed,variable=\y,BLUE,line width=1.5pt] plot ({0.4327},{\y});
%   %
%   \draw[scale=2,domain=0.4327:2,smooth,variable=\x,BLUE,line width=1.5pt] plot ({\x},{0.8744});
  %
  %
  %
%  \draw[scale=1.5,domain=-2:2,smooth,variable=\x,RED,line width=1pt] plot ({\x},{0.1145*\x*\x*\x+0.2352*\x});
  \end{tikzpicture}
   %\caption{Piecewise constant (blue), \(\tanh\) (green) and cubic (red) function with equal \((a_1, a_2, \nu)\).}
   \captionof{figure}{Piecewise constant (green) versus cubic (blue) function with equal \((a_1, a_2, \nu)\).}
  \label{fig:illustration-nonlinearities}
\end{minipage}
\end{figure}

where \(r \equiv \frac{1-\erf(s_+)}{1 + \erf(s_-)}\). Figure~\ref{fig:illustration-nonlinearities} displays \(f\) given in \eqref{eq:def-piecewise-func} together with the cubic function \(c_3 x^3 + c_2 (x^2 -1) + c_1 x\) sharing the same Hermite coefficients \((a_1, a_2, \nu)\). 

The class of equivalence of kernel functions induced by this mapping is quite unlike that raised in \cite{el2010spectrum} or \cite{couillet2016kernel} in the ``improper'' scaling \(f(\x_i^\T\x_j/p)\) regime. While in the latter, functions \(f(x)\) of the same class of equivalence are those having common \(f'(0)\) and \(f''(0)\) values, in the present case, these functions may have no similar local behavior (as shown in the example of Figure~\ref{fig:illustration-nonlinearities}).

For the piecewise constant function defined in \eqref{eq:def-piecewise-func} and the associated cubic function having the same \( (a_1, a_2, \nu) \), a close match is observed for both eigenvalues and top eigenvectors of \(\K\) in Figure~\ref{fig:piecewise-versus-cubic}, with gains in both storage size and computational time displayed in Table~\ref{tab:size-time-f}.

%%% comparison between cubic and piecewise function
\begin{figure}[htb]
\centering
    \begin{subfigure}[c]{0.25\textwidth}
        \centering
        \begin{tikzpicture}[font=\footnotesize,spy using outlines]
        \renewcommand{\axisdefaulttryminticks}{4} 
        \pgfplotsset{every major grid/.style={densely dashed}}       
        \tikzstyle{every axis y label}+=[yshift=-10pt] 
        \tikzstyle{every axis x label}+=[yshift=5pt]
        \pgfplotsset{every axis legend/.style={cells={anchor=west},fill=white,
        at={(0.98,1)}, anchor=north east, font=\footnotesize }}
        \begin{axis}[
        width=1.35\textwidth,
        height=1.2\textwidth,
        xmin=-.8,
        ymin=0,
        xmax=1.2,
        ymax=1.17,
        ytick={0,0.2,0.4,0.6,0.8},
        yticklabels = {},
        bar width=2pt,
        grid=major,
        ymajorgrids=false,
        scaled ticks=true,
        xlabel={},
        ylabel={}
        ]
        \addplot+[ybar,mark=none,color=white,fill=GREEN,area legend] coordinates{
        %(-0.703494,0.000000)(-0.665548,0.000000)(-0.627602,0.000000)(-0.589656,0.115810)(-0.551711,0.373167)(-0.513765,0.566185)(-0.475819,0.707731)(-0.437873,0.784938)(-0.399927,0.913616)(-0.361981,0.926484)(-0.324035,1.016559)(-0.286089,1.042294)(-0.248144,1.068030)(-0.210198,1.106634)(-0.172252,1.093766)(-0.134306,1.106634)(-0.096360,1.119501)(-0.058414,1.119501)(-0.020468,1.068030)(0.017478,1.080898)(0.055424,1.055162)(0.093369,1.003691)(0.131315,0.990823)(0.169261,0.952220)(0.207207,0.900748)(0.245153,0.836409)(0.283099,0.823541)(0.321045,0.759202)(0.358991,0.681995)(0.396936,0.669127)(0.434882,0.591920)(0.472828,0.527581)(0.510774,0.463242)(0.548720,0.373167)(0.586666,0.321696)(0.624612,0.167282)(0.662558,0.012868)(0.700504,0.000000)(0.738449,0.000000)(0.776395,0.000000)(0.814341,0.000000)(0.852287,0.000000)(0.890233,0.000000)(0.928179,0.012868)(0.966125,0.000000)(1.004071,0.000000)(1.042017,0.000000)(1.079962,0.000000)(1.117908,0.000000)(1.155854,0.000000) %% equivalent to cubic
        (-0.695432,0.000000)(-0.657052,0.000000)(-0.618673,0.000000)(-0.580293,0.203560)(-0.541914,0.458009)(-0.503534,0.610679)(-0.465155,0.750626)(-0.426776,0.801516)(-0.388396,0.916018)(-0.350017,0.954186)(-0.311637,1.017798)(-0.273258,1.055965)(-0.234878,1.055965)(-0.196499,1.106855)(-0.158120,1.119578)(-0.119740,1.106855)(-0.081361,1.106855)(-0.042981,1.106855)(-0.004602,1.068688)(0.033778,1.068688)(0.072157,1.030520)(0.110536,0.966908)(0.148916,0.979630)(0.187295,0.916018)(0.225675,0.877851)(0.264054,0.826961)(0.302434,0.788793)(0.340813,0.725181)(0.379192,0.674291)(0.417572,0.623401)(0.455951,0.547066)(0.494331,0.496176)(0.532710,0.407119)(0.571090,0.330784)(0.609469,0.267172)(0.647848,0.076335)(0.686228,0.000000)(0.724607,0.000000)(0.762987,0.000000)(0.801366,0.000000)(0.839745,0.000000)(0.878125,0.000000)(0.916504,0.000000)(0.954884,0.012722)(0.993263,0.000000)(1.031643,0.000000)(1.070022,0.000000)(1.108401,0.000000)(1.146781,0.000000)(1.185160,0.000000)
        %(-0.662075,0.000000)(-0.626293,0.000000)(-0.590512,0.000000)(-0.554731,0.382098)(-0.518950,0.695964)(-0.483169,0.846073)(-0.447388,0.941598)(-0.411607,1.023476)(-0.375826,1.078061)(-0.340045,1.078061)(-0.304264,1.078061)(-0.268482,1.132647)(-0.232701,1.119000)(-0.196920,1.064415)(-0.161139,1.078061)(-0.125358,1.023476)(-0.089577,1.050769)(-0.053796,0.982537)(-0.018015,0.968891)(0.017766,0.941598)(0.053547,0.873366)(0.089329,0.887012)(0.125110,0.846073)(0.160891,0.818781)(0.196672,0.750549)(0.232453,0.764195)(0.268234,0.682317)(0.304015,0.682317)(0.339796,0.655025)(0.375577,0.586793)(0.411358,0.559500)(0.447140,0.532207)(0.482921,0.491268)(0.518702,0.463976)(0.554483,0.436683)(0.590264,0.368451)(0.626045,0.327512)(0.661826,0.300220)(0.697607,0.231988)(0.733388,0.136463)(0.769169,0.054585)(0.804951,0.000000)(0.840732,0.000000)(0.876513,0.013646)(0.912294,0.000000)(0.948075,0.000000)(0.983856,0.000000)(1.019637,0.000000)(1.055418,0.000000)(1.091199,0.000000) %% equivalent to tanh
        };
        \addplot[smooth,RED,line width=.5pt] plot coordinates{
        (-0.695432,0.000024)(-0.657052,0.000040)(-0.618673,0.000284)(-0.580293,0.453299)(-0.541914,0.630014)(-0.503534,0.752008)(-0.465155,0.843796)(-0.426776,0.915040)(-0.388396,0.970674)(-0.350017,1.013776)(-0.311637,1.046358)(-0.273258,1.069854)(-0.234878,1.085435)(-0.196499,1.093930)(-0.158120,1.096093)(-0.119740,1.092453)(-0.081361,1.083547)(-0.042981,1.069876)(-0.004602,1.051778)(0.033778,1.029612)(0.072157,1.003707)(0.110536,0.974355)(0.148916,0.941727)(0.187295,0.906039)(0.225675,0.867450)(0.264054,0.825965)(0.302434,0.781607)(0.340813,0.734256)(0.379192,0.683708)(0.417572,0.629556)(0.455951,0.571107)(0.494331,0.507117)(0.532710,0.435406)(0.571090,0.351717)(0.609469,0.244075)(0.647848,0.000150)(0.686228,0.000025)(0.724607,0.000016)(0.762987,0.000012)(0.801366,0.000010)(0.839745,0.000008)(0.878125,0.000007)(0.916504,0.000006)(0.954884,0.000005)(0.993263,0.000005)(1.031643,0.000004)(1.070022,0.000004)(1.108401,0.000003)(1.146781,0.000003)(1.185160,0.000003)
        };
        \coordinate (spike) at (0.95,0.01);
        \coordinate (spike_spy) at (.9,0.5);
        \end{axis}
        \begin{scope}
          \spy[black!50!white,size=.8cm,circle,connect spies,magnification=8] on (spike) in node [fill=none] at (spike_spy);
        \end{scope}
        % \end{axis}
        % \begin{scope}
        %   \spy[black!50!white,size=.8cm,circle,connect spies,magnification=10] on (2.75,0.02) in node [fill=none] at (2.5,1.5);
        % \end{scope}
        \end{tikzpicture}
        \caption{Piecewise constant}
    \end{subfigure}
    \hfill{}%
    %
    %
    %
    %
    %
    %%% cubic function
    \begin{subfigure}[c]{0.25\textwidth}
        \centering
        \begin{tikzpicture}[font=\footnotesize,spy using outlines]
        \renewcommand{\axisdefaulttryminticks}{4} 
        \pgfplotsset{every major grid/.style={densely dashed}}       
        \tikzstyle{every axis y label}+=[yshift=-10pt] 
        \tikzstyle{every axis x label}+=[yshift=5pt]
        \pgfplotsset{every axis legend/.style={cells={anchor=west},fill=white,
        at={(0.98,1)}, anchor=north east, font=\footnotesize }}
        \begin{axis}[
        width=1.35\textwidth,
        height=1.2\textwidth,
        xmin=-.8,
        ymin=0,
        xmax=1.2,
        ymax=1.2,
        ytick={0,0.2,0.4,0.6,0.8},
        yticklabels = {},
        bar width=2pt,
        grid=major,
        ymajorgrids=false,
        scaled ticks=true,
        xlabel={},
        ylabel={}
        ]
        \addplot+[ybar,mark=none,color=white,fill=BLUE!60!white,area legend] coordinates{
        %(-0.767323,0.000000)(-0.726194,0.000000)(-0.685066,0.000000)(-0.643938,0.166210)(-0.602810,0.296804)(-0.561681,0.415526)(-0.520553,0.510503)(-0.479425,0.593608)(-0.438297,0.688586)(-0.397169,0.759819)(-0.356040,0.854796)(-0.314912,0.914157)(-0.273784,0.985390)(-0.232656,1.056623)(-0.191527,1.127856)(-0.150399,1.139728)(-0.109271,1.163472)(-0.068143,1.163472)(-0.027014,1.163472)(0.014114,1.127856)(0.055242,1.044751)(0.096370,1.044751)(0.137499,0.961646)(0.178627,0.914157)(0.219755,0.854796)(0.260883,0.795435)(0.302011,0.724202)(0.343140,0.652969)(0.384268,0.605481)(0.425396,0.534248)(0.466524,0.486759)(0.507653,0.415526)(0.548781,0.368037)(0.589909,0.308676)(0.631037,0.249316)(0.672166,0.166210)(0.713294,0.047489)(0.754422,0.000000)(0.795550,0.000000)(0.836678,0.000000)(0.877807,0.000000)(0.918935,0.000000)(0.960063,0.000000)(1.001191,0.011872)(1.042320,0.000000)(1.083448,0.000000)(1.124576,0.000000)(1.165704,0.000000)(1.206833,0.000000)(1.247961,0.000000)
        (-0.773288,0.000000)(-0.731589,0.000000)(-0.689890,0.000000)(-0.648191,0.175645)(-0.606492,0.292742)(-0.564794,0.421549)(-0.523095,0.480097)(-0.481396,0.608904)(-0.439697,0.679162)(-0.397998,0.761130)(-0.356299,0.819678)(-0.314600,0.948485)(-0.272901,0.995324)(-0.231202,1.053872)(-0.189504,1.100711)(-0.147805,1.147550)(-0.106106,1.170969)(-0.064407,1.147550)(-0.022708,1.135840)(0.018991,1.100711)(0.060690,1.077292)(0.102389,0.995324)(0.144088,0.983614)(0.185787,0.901646)(0.227485,0.819678)(0.269184,0.784549)(0.310883,0.702582)(0.352582,0.655743)(0.394281,0.585485)(0.435980,0.503517)(0.477679,0.456678)(0.519378,0.421549)(0.561077,0.339581)(0.602775,0.304452)(0.644474,0.222484)(0.686173,0.140516)(0.727872,0.035129)(0.769571,0.000000)(0.811270,0.000000)(0.852969,0.000000)(0.894668,0.000000)(0.936367,0.000000)(0.978066,0.000000)(1.019764,0.011710)(1.061463,0.000000)(1.103162,0.000000)(1.144861,0.000000)(1.186560,0.000000)(1.228259,0.000000)(1.269958,0.000000)
        };
        \addplot[smooth,RED,line width=.5pt] plot coordinates{
        %(-0.767323,0.000010)(-0.726194,0.000012)(-0.685066,0.000016)(-0.643938,0.000023)(-0.602810,0.000039)(-0.561681,0.000925)(-0.520553,0.484138)(-0.479425,0.664061)(-0.438297,0.787358)(-0.397169,0.878771)(-0.356040,0.948199)(-0.314912,1.000855)(-0.273784,1.039915)(-0.232656,1.067535)(-0.191527,1.085300)(-0.150399,1.094418)(-0.109271,1.095804)(-0.068143,1.090240)(-0.027014,1.078466)(0.014114,1.061007)(0.055242,1.038429)(0.096370,1.011097)(0.137499,0.979435)(0.178627,0.943768)(0.219755,0.904329)(0.260883,0.861326)(0.302011,0.814827)(0.343140,0.764767)(0.384268,0.711006)(0.425396,0.653079)(0.466524,0.590308)(0.507653,0.521360)(0.548781,0.443662)(0.589909,0.351930)(0.631037,0.230430)(0.672166,0.000059)(0.713294,0.000021)(0.754422,0.000014)(0.795550,0.000011)(0.836678,0.000009)(0.877807,0.000007)(0.918935,0.000006)(0.960063,0.000005)(1.001191,0.000005)(1.042320,0.000004)(1.083448,0.000004)(1.124576,0.000003)(1.165704,0.000003)(1.206833,0.000003)(1.247961,0.000003)
        (-0.773288,0.000014)(-0.731589,0.000018)(-0.689890,0.000027)(-0.648191,0.000053)(-0.606492,0.291001)(-0.564794,0.553186)(-0.523095,0.707077)(-0.481396,0.817062)(-0.439697,0.899939)(-0.397998,0.963284)(-0.356299,1.011428)(-0.314600,1.046984)(-0.272901,1.071935)(-0.231202,1.087620)(-0.189504,1.095144)(-0.147805,1.095386)(-0.106106,1.089029)(-0.064407,1.076745)(-0.022708,1.059083)(0.018991,1.036452)(0.060690,1.009305)(0.102389,0.978043)(0.144088,0.942873)(0.185787,0.904100)(0.227485,0.861885)(0.269184,0.816308)(0.310883,0.767301)(0.352582,0.714703)(0.394281,0.658208)(0.435980,0.597039)(0.477679,0.530099)(0.519378,0.455167)(0.561077,0.367738)(0.602775,0.255625)(0.644474,0.000160)(0.686173,0.000024)(0.727872,0.000015)(0.769571,0.000011)(0.811270,0.000009)(0.852969,0.000007)(0.894668,0.000006)(0.936367,0.000005)(0.978066,0.000005)(1.019764,0.000004)(1.061463,0.000004)(1.103162,0.000003)(1.144861,0.000003)(1.186560,0.000003)(1.228259,0.000003)(1.269958,0.000002)
        };
        \coordinate (spike) at (1.02,0.01);
        \coordinate (spike_spy) at (.9,0.5);
        \end{axis}
        \begin{scope}
          \spy[black!50!white,size=.8cm,circle,connect spies,magnification=8] on (spike) in node [fill=none] at (spike_spy);
        \end{scope}
        % \end{axis}
        % \begin{scope}
        %   \spy[black!50!white,size=.8cm,circle,connect spies,magnification=10] on (2.85,0.02) in node [fill=none] at (2.5,1.5);
        % \end{scope}
        \end{tikzpicture}
        \caption{Cubic polynomial}
    \end{subfigure}%
    \begin{subfigure}[c]{0.48\textwidth}
    \centering
    \begin{tikzpicture}[font=\footnotesize,spy using outlines]
        \renewcommand{\axisdefaulttryminticks}{4} 
        \pgfplotsset{every major grid/.style={densely dashed}}       
        \tikzstyle{every axis y label}+=[yshift=-10pt] 
        \tikzstyle{every axis x label}+=[yshift=5pt]
        \pgfplotsset{every axis legend/.style={cells={anchor=west},fill=white,
        at={(0.98,1)}, anchor=north east, font=\footnotesize}}
        \begin{axis}[
        width=1\textwidth,
        height=.7\textwidth,
        xmin=0,
        ymin=-0.05,
        xmax=410,
        ymax=0.03,
        ytick={-0.05, 0, 0.05},
        xticklabels = {},
        grid=major,
        ymajorgrids=false,
        scaled ticks=true,
        xlabel={},
        ylabel={}
        ]
        \addplot[densely dashed,GREEN,line width=.5pt] coordinates{
        (1,0.012815)(2,0.012921)(3,-0.004221)(4,0.000062)(5,0.016318)(6,0.009867)(7,0.013321)(8,0.012131)(9,0.005399)(10,-0.004032)(11,0.014383)(12,0.007106)(13,0.018542)(14,0.010648)(15,0.012957)(16,0.005735)(17,-0.006320)(18,0.015816)(19,0.020323)(20,-0.016641)(21,0.002150)(22,0.013875)(23,0.012995)(24,-0.000736)(25,0.009843)(26,0.019045)(27,0.010060)(28,0.001422)(29,-0.008873)(30,-0.010849)(31,0.007342)(32,-0.003722)(33,0.020387)(34,0.011522)(35,0.003709)(36,0.014791)(37,0.009291)(38,0.014614)(39,0.011548)(40,-0.000418)(41,0.011474)(42,-0.013212)(43,0.020337)(44,0.006436)(45,0.004683)(46,-0.008871)(47,0.019684)(48,0.001116)(49,0.014176)(50,0.002939)(51,-0.009900)(52,0.004564)(53,0.010583)(54,0.011497)(55,0.006415)(56,0.018481)(57,0.012916)(58,0.012413)(59,0.003951)(60,0.016971)(61,0.008296)(62,-0.007710)(63,-0.003791)(64,0.002215)(65,0.016340)(66,0.004434)(67,0.002971)(68,0.012980)(69,0.004579)(70,-0.003216)(71,0.011096)(72,-0.002693)(73,0.004766)(74,0.010812)(75,0.018249)(76,0.023383)(77,0.014529)(78,0.010797)(79,0.003643)(80,0.018448)(81,-0.004028)(82,0.009879)(83,0.002050)(84,-0.005329)(85,0.011545)(86,0.013843)(87,0.016424)(88,0.012071)(89,-0.005871)(90,0.022184)(91,0.010940)(92,0.005131)(93,-0.004054)(94,0.018736)(95,0.010148)(96,0.016484)(97,0.010023)(98,0.012485)(99,-0.003344)(100,-0.002414)(101,0.007962)(102,-0.003234)(103,-0.001777)(104,0.014320)(105,0.007648)(106,0.003084)(107,-0.009217)(108,-0.006849)(109,-0.003477)(110,0.004253)(111,0.007988)(112,0.011015)(113,0.001262)(114,0.021348)(115,0.019116)(116,0.001626)(117,0.016761)(118,0.012235)(119,0.003651)(120,0.011049)(121,-0.008965)(122,0.006419)(123,0.018708)(124,0.004367)(125,0.008355)(126,0.010301)(127,0.014212)(128,0.009151)(129,0.007221)(130,0.000796)(131,-0.001584)(132,0.015890)(133,0.004699)(134,0.016153)(135,-0.010707)(136,-0.003650)(137,-0.004758)(138,0.011223)(139,0.011442)(140,0.020056)(141,-0.005985)(142,-0.000996)(143,0.012865)(144,0.010878)(145,0.023698)(146,0.000848)(147,-0.003812)(148,-0.012328)(149,0.002512)(150,-0.000389)(151,-0.011022)(152,0.014179)(153,0.006320)(154,0.012355)(155,0.009310)(156,0.001486)(157,0.007441)(158,0.012634)(159,0.022369)(160,0.006511)(161,0.008971)(162,0.013229)(163,0.000871)(164,0.011653)(165,0.006432)(166,0.010200)(167,0.002211)(168,0.001560)(169,-0.005331)(170,0.017052)(171,0.008849)(172,0.002661)(173,-0.007047)(174,-0.010507)(175,0.003490)(176,0.011689)(177,-0.006515)(178,0.003514)(179,0.015022)(180,0.003016)(181,0.009305)(182,0.007067)(183,0.000579)(184,0.008545)(185,0.003015)(186,0.012926)(187,0.017481)(188,-0.003549)(189,0.010807)(190,-0.013168)(191,-0.005466)(192,-0.006998)(193,0.011133)(194,0.000812)(195,0.010933)(196,0.000722)(197,0.009851)(198,-0.000888)(199,0.008899)(200,0.001905)(201,0.012632)(202,0.012290)(203,0.006777)(204,0.008597)(205,0.008729)(206,-0.041620)(207,-0.012128)(208,-0.021690)(209,-0.034106)(210,-0.011277)(211,-0.028288)(212,-0.023449)(213,-0.043184)(214,-0.024925)(215,0.006442)(216,-0.037266)(217,-0.015097)(218,-0.028398)(219,-0.021920)(220,-0.032176)(221,-0.041151)(222,-0.039924)(223,-0.030302)(224,-0.015817)(225,-0.024588)(226,-0.042787)(227,-0.031620)(228,-0.033721)(229,-0.042946)(230,-0.039722)(231,-0.028867)(232,-0.015314)(233,-0.013575)(234,-0.042645)(235,-0.021709)(236,-0.037688)(237,-0.021859)(238,-0.035809)(239,-0.029556)(240,-0.041571)(241,-0.033309)(242,-0.020114)(243,-0.018279)(244,-0.008171)(245,-0.047024)(246,-0.032769)(247,-0.023435)(248,-0.027948)(249,-0.013233)(250,-0.021650)(251,-0.030571)(252,-0.036989)(253,-0.024759)(254,-0.016090)(255,-0.016789)(256,-0.028734)(257,-0.026973)(258,-0.024661)(259,-0.025011)(260,-0.021533)(261,-0.037831)(262,-0.015651)(263,-0.040873)(264,-0.021642)(265,-0.018203)(266,-0.000593)(267,-0.045232)(268,-0.024090)(269,-0.039502)(270,-0.026161)(271,-0.016852)(272,-0.043255)(273,-0.013614)(274,-0.016781)(275,-0.029071)(276,-0.019855)(277,-0.029933)(278,-0.034187)(279,-0.033428)(280,-0.012436)(281,-0.035311)(282,-0.020071)(283,-0.017127)(284,-0.037260)(285,-0.013074)(286,-0.023107)(287,-0.015786)(288,-0.029411)(289,-0.038258)(290,-0.045727)(291,-0.008412)(292,-0.025826)(293,-0.040474)(294,-0.037216)(295,-0.025181)(296,-0.027724)(297,-0.024930)(298,-0.025337)(299,-0.030178)(300,-0.038002)(301,-0.027394)(302,-0.029228)(303,-0.064109)(304,-0.030276)(305,-0.029698)(306,-0.028912)(307,-0.027763)(308,-0.056315)(309,-0.040484)(310,-0.033805)(311,-0.016977)(312,-0.015798)(313,-0.046205)(314,-0.026865)(315,-0.014016)(316,-0.026015)(317,-0.039719)(318,-0.019074)(319,-0.026030)(320,-0.027753)(321,-0.016964)(322,-0.047414)(323,-0.023450)(324,-0.025117)(325,-0.030094)(326,-0.037989)(327,-0.043954)(328,-0.013012)(329,-0.021748)(330,-0.034182)(331,-0.023233)(332,-0.036615)(333,-0.032165)(334,-0.028839)(335,-0.029426)(336,-0.011019)(337,-0.043641)(338,-0.034059)(339,-0.052494)(340,-0.025637)(341,-0.045627)(342,-0.034462)(343,-0.026409)(344,-0.024289)(345,-0.030404)(346,-0.039596)(347,-0.013341)(348,-0.039855)(349,-0.031225)(350,-0.009753)(351,-0.036281)(352,-0.009981)(353,-0.022981)(354,-0.017320)(355,-0.015644)(356,-0.037773)(357,0.009734)(358,-0.029403)(359,-0.039652)(360,-0.022990)(361,-0.022511)(362,-0.055080)(363,-0.036200)(364,-0.019307)(365,-0.025460)(366,-0.041129)(367,-0.031976)(368,-0.019423)(369,-0.022168)(370,-0.030454)(371,-0.059768)(372,-0.023882)(373,-0.022470)(374,-0.021678)(375,-0.035022)(376,-0.020465)(377,-0.014262)(378,-0.020811)(379,-0.011296)(380,-0.014990)(381,-0.029225)(382,-0.015879)(383,-0.037241)(384,-0.029297)(385,-0.029699)(386,-0.029363)(387,-0.026494)(388,-0.020665)(389,-0.044119)(390,-0.030998)(391,-0.021618)(392,-0.056229)(393,-0.021532)(394,-0.046573)(395,-0.015026)(396,-0.023361)(397,-0.005799)(398,-0.050788)(399,-0.018350)(400,-0.022673)(401,-0.032463)(402,-0.034356)(403,-0.027358)(404,-0.039429)(405,-0.014939)(406,-0.028753)(407,-0.020149)(408,-0.021645)(409,-0.014800)(410,-0.014355)
        };
        %%%
        \addplot[smooth,BLUE!60!white,line width=.5pt] coordinates{
        (1,0.011311)(2,0.009021)(3,0.005284)(4,-0.000844)(5,0.013558)(6,0.005207)(7,0.009430)(8,-0.002252)(9,-0.000425)(10,-0.000707)(11,0.016655)(12,0.008468)(13,0.020167)(14,0.008641)(15,0.017211)(16,0.008312)(17,-0.001284)(18,0.014168)(19,0.014221)(20,-0.010357)(21,-0.003784)(22,0.009141)(23,0.017482)(24,0.001938)(25,0.000934)(26,0.014411)(27,0.013643)(28,-0.005883)(29,0.000946)(30,-0.001582)(31,0.009119)(32,0.004759)(33,0.014458)(34,0.011538)(35,0.010134)(36,0.007809)(37,0.006599)(38,0.007910)(39,0.009649)(40,-0.000720)(41,0.006049)(42,-0.008181)(43,0.012125)(44,0.010313)(45,0.010072)(46,-0.009140)(47,0.010911)(48,-0.002939)(49,0.009315)(50,0.005131)(51,-0.008190)(52,0.003232)(53,0.015840)(54,0.013588)(55,0.003710)(56,0.009177)(57,0.015818)(58,0.013890)(59,0.003708)(60,0.014787)(61,0.007380)(62,-0.005815)(63,-0.004768)(64,0.005518)(65,0.007963)(66,0.006673)(67,0.006831)(68,0.000911)(69,-0.001629)(70,-0.005460)(71,0.008335)(72,-0.003328)(73,0.007443)(74,0.007238)(75,0.021429)(76,0.012032)(77,0.009627)(78,0.005176)(79,0.004473)(80,0.010141)(81,-0.006079)(82,0.011692)(83,0.007963)(84,-0.005175)(85,0.012930)(86,0.011546)(87,0.013013)(88,0.013579)(89,-0.004507)(90,0.020259)(91,0.008653)(92,0.003070)(93,0.003824)(94,0.011227)(95,0.011096)(96,0.009896)(97,0.012890)(98,0.003282)(99,-0.003216)(100,0.000052)(101,0.009565)(102,-0.006036)(103,-0.007288)(104,0.012598)(105,0.010239)(106,0.000963)(107,-0.010173)(108,-0.000358)(109,-0.002377)(110,0.001538)(111,0.011157)(112,0.000650)(113,-0.001098)(114,0.021187)(115,0.016416)(116,0.008254)(117,0.012501)(118,0.008500)(119,0.007228)(120,0.005503)(121,-0.008377)(122,0.009375)(123,0.012397)(124,0.004067)(125,0.003920)(126,0.008550)(127,0.009379)(128,0.012788)(129,0.004720)(130,-0.002657)(131,0.002081)(132,0.014114)(133,0.005218)(134,0.008918)(135,-0.009312)(136,-0.005828)(137,-0.006351)(138,0.004769)(139,0.014836)(140,0.013485)(141,-0.002013)(142,-0.004497)(143,0.008702)(144,0.001979)(145,0.021433)(146,-0.003884)(147,-0.003318)(148,-0.006910)(149,0.005084)(150,-0.001810)(151,-0.007901)(152,0.005751)(153,0.010070)(154,0.011560)(155,0.007418)(156,-0.000011)(157,0.010278)(158,0.015677)(159,0.019509)(160,0.006347)(161,0.009056)(162,0.009274)(163,0.004073)(164,0.011074)(165,0.003367)(166,0.007342)(167,0.005019)(168,-0.001437)(169,-0.009353)(170,0.015127)(171,0.005268)(172,0.008719)(173,-0.003476)(174,-0.002078)(175,0.002261)(176,0.007317)(177,-0.001819)(178,0.002063)(179,0.010775)(180,0.005036)(181,0.009678)(182,0.009278)(183,-0.001662)(184,0.005756)(185,0.009244)(186,0.008882)(187,0.012998)(188,-0.002808)(189,0.013006)(190,-0.011791)(191,0.003734)(192,-0.009850)(193,0.003641)(194,0.006544)(195,0.007449)(196,0.001317)(197,0.009827)(198,-0.000444)(199,-0.003215)(200,-0.002539)(201,0.011963)(202,0.013692)(203,0.001981)(204,0.008394)(205,0.005342)(206,-0.037194)(207,-0.009073)(208,-0.022479)(209,-0.041239)(210,-0.008380)(211,-0.035158)(212,-0.022177)(213,-0.039973)(214,-0.018233)(215,0.006354)(216,-0.031918)(217,-0.011118)(218,-0.032130)(219,-0.023863)(220,-0.029026)(221,-0.039271)(222,-0.040366)(223,-0.044206)(224,-0.028423)(225,-0.017743)(226,-0.032210)(227,-0.019835)(228,-0.038506)(229,-0.038734)(230,-0.037778)(231,-0.034989)(232,-0.030044)(233,-0.020824)(234,-0.046296)(235,-0.024136)(236,-0.038120)(237,-0.028714)(238,-0.033501)(239,-0.030295)(240,-0.033226)(241,-0.026821)(242,-0.016038)(243,-0.024587)(244,-0.000730)(245,-0.041639)(246,-0.035226)(247,-0.017668)(248,-0.019755)(249,-0.023481)(250,-0.019978)(251,-0.037462)(252,-0.035525)(253,-0.021492)(254,-0.010866)(255,-0.015736)(256,-0.026037)(257,-0.028819)(258,-0.026046)(259,-0.036062)(260,-0.023288)(261,-0.028476)(262,-0.006612)(263,-0.039045)(264,-0.030679)(265,-0.021481)(266,-0.006981)(267,-0.034598)(268,-0.019759)(269,-0.051889)(270,-0.025301)(271,-0.015632)(272,-0.047606)(273,-0.012410)(274,-0.022877)(275,-0.025627)(276,-0.025649)(277,-0.031645)(278,-0.022951)(279,-0.032483)(280,-0.009667)(281,-0.027524)(282,-0.018763)(283,-0.013211)(284,-0.038684)(285,-0.007526)(286,-0.034519)(287,-0.020966)(288,-0.013562)(289,-0.039248)(290,-0.039423)(291,-0.013077)(292,-0.017334)(293,-0.020129)(294,-0.041345)(295,-0.024952)(296,-0.023374)(297,-0.026936)(298,-0.033756)(299,-0.026671)(300,-0.045737)(301,-0.026285)(302,-0.036249)(303,-0.059300)(304,-0.030298)(305,-0.031821)(306,-0.014360)(307,-0.029919)(308,-0.050790)(309,-0.047298)(310,-0.021804)(311,-0.022585)(312,-0.007368)(313,-0.034950)(314,-0.024276)(315,-0.017398)(316,-0.024584)(317,-0.033994)(318,-0.009230)(319,-0.023081)(320,-0.009388)(321,-0.025689)(322,-0.038571)(323,-0.018657)(324,-0.024608)(325,-0.018159)(326,-0.054161)(327,-0.043067)(328,-0.021305)(329,-0.029310)(330,-0.033454)(331,-0.031599)(332,-0.046912)(333,-0.039133)(334,-0.031060)(335,-0.024895)(336,-0.009778)(337,-0.037285)(338,-0.043530)(339,-0.046492)(340,-0.030859)(341,-0.045619)(342,-0.030971)(343,-0.026585)(344,-0.018563)(345,-0.035868)(346,-0.043397)(347,-0.016521)(348,-0.043670)(349,-0.026780)(350,-0.008642)(351,-0.036398)(352,-0.004978)(353,-0.026799)(354,-0.020594)(355,-0.014010)(356,-0.030731)(357,0.005371)(358,-0.027262)(359,-0.028872)(360,-0.026002)(361,-0.018241)(362,-0.061381)(363,-0.028655)(364,-0.028173)(365,-0.026483)(366,-0.045973)(367,-0.043928)(368,-0.021214)(369,-0.024567)(370,-0.036191)(371,-0.061603)(372,-0.010353)(373,-0.037536)(374,-0.011122)(375,-0.044948)(376,-0.024212)(377,-0.014938)(378,-0.017264)(379,-0.022537)(380,-0.013104)(381,-0.034175)(382,-0.018198)(383,-0.027663)(384,-0.026378)(385,-0.029577)(386,-0.034679)(387,-0.023154)(388,-0.036460)(389,-0.039515)(390,-0.035077)(391,-0.007940)(392,-0.030699)(393,-0.027232)(394,-0.045100)(395,-0.016782)(396,-0.022862)(397,-0.010336)(398,-0.054896)(399,-0.023645)(400,-0.017023)(401,-0.031650)(402,-0.031017)(403,-0.013097)(404,-0.041676)(405,-0.005477)(406,-0.024241)(407,-0.040387)(408,-0.019767)(409,-0.012836)(410,-0.009213)
        };
        \end{axis}
        \end{tikzpicture}
         %\caption{Piecewise constant (blue) vs \(\tanh\) (green)}
         \caption{Cubic (blue) versus piecewise (green) function}
    \end{subfigure}
    %\caption{Eigenvalue distribution and top eigenvectors of \(\K\) for the hyperbolic tangent function (in blue, with respect to Bernoulli distribution with zero mean and unit variance) and the associated piecewise constant function (in blue, performed on standard Gaussian distribution) in \eqref{eq:def-piecewise-func} with the same \((a_1, a_2, \nu)\), in the setting of Figure~\ref{fig:validation-approx}.}
    \caption{Eigenvalue distribution and top eigenvectors of \(\K\) for the piecewise constant function (in green) and the associated cubic function (in blue) with the same \((a_1, a_2, \nu)\), performed on Bernoulli distribution with zero mean and unit variance, in the setting of Figure~\ref{fig:validation-approx}.}
    \label{fig:piecewise-versus-cubic} 
\end{figure}
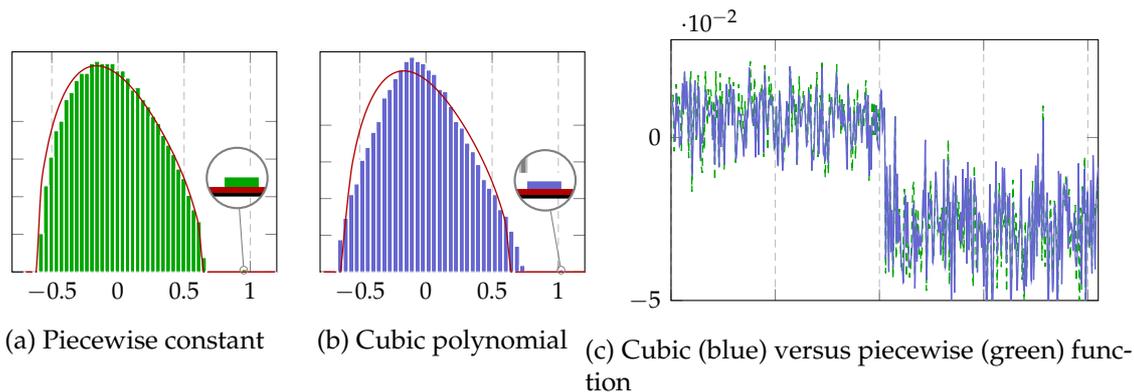

%{\RED  \(\tanh\) nonlinearity: \(a_1 = 0.605706\), \(a_2 = 0\) and \(\nu = 0.394294\).}

\section{Concluding remarks}
\label{sec:conclusion}

We have shown that inner-product kernel matrices \(\sqrt{p}\K=f(\x_i^\T\x_j/\sqrt{p})\) with \(\x_i=\bmu_a+(\I_p+\E_a)^{\frac12}\z_i\), \(a\in\{1,2\}\), asymptotically behave as a spiked random matrix model which spectrally only depends on three defining parameters of \(f\). Turning \(\I_p\) into a generic \(\C\) covariance is more technically challenging, breaking most of the orthogonality properties of the orthogonal polynomial approach of the proofs, but a needed extension of the result.

Interestingly, this study can be compared to analyses in neural networks (see e.g.,  \cite{pennington2017nonlinear,benigni2019eigenvalue}) where it has been shown that in the case of sub-Gaussian entries for both random layer \(\W\) and input data \(\X\) the (limiting) spectrum of the Gram matrix \(f(\W \X)f(\W \X)^\T\) ( \(f\) understood entry-wise) is uniquely determined by the same \((a_1, \nu)\) coefficients. Our results may then be adapted to an improved understanding of classification in random neural networks. 
%{\BLUE When deep networks are considered, it was conjectured in \cite{pennington2017nonlinear} and then proved in \cite{benigni2019eigenvalue} that, }

\medskip

\small
\bibliographystyle{alpha}
\bibliography{liao}

\newcommand{\etalchar}[1]{$^{#1}$}
\begin{thebibliography}{ABDH{\etalchar{+}}18}

\bibitem[AAR00]{andrews2000special}
George~E Andrews, Richard Askey, and Ranjan Roy.
\newblock {\em Special functions}, volume~71.
\newblock Cambridge university press, 2000.

\bibitem[ABDH{\etalchar{+}}18]{ashtiani2018nearly}
Hassan Ashtiani, Shai Ben-David, Nicholas Harvey, Christopher Liaw, Abbas
  Mehrabian, and Yaniv Plan.
\newblock Nearly tight sample complexity bounds for learning mixtures of
  gaussians via sample compression schemes.
\newblock In {\em Advances in Neural Information Processing Systems}, pages
  3412--3421, 2018.

\bibitem[AS65]{abramowitz1965handbook}
Milton Abramowitz and Irene~A Stegun.
\newblock {\em Handbook of mathematical functions: with formulas, graphs, and
  mathematical tables}, volume~55.
\newblock Courier Corporation, 1965.

\bibitem[BAP05]{baik2005phase}
Jinho Baik, G{\'e}rard~Ben Arous, and Sandrine P{\'e}ch{\'e}.
\newblock Phase transition of the largest eigenvalue for nonnull complex sample
  covariance matrices.
\newblock {\em The Annals of Probability}, 33(5):1643--1697, 2005.

\bibitem[BGN11]{benaych2011eigenvalues}
Florent Benaych-Georges and Raj~Rao Nadakuditi.
\newblock The eigenvalues and eigenvectors of finite, low rank perturbations of
  large random matrices.
\newblock {\em Advances in Mathematics}, 227(1):494--521, 2011.

\bibitem[Bil12]{billingsley2012probability}
Patrick Billingsley.
\newblock {\em Probability and measure}.
\newblock John Wiley \& Sons, third edition, 2012.

\bibitem[BP19]{benigni2019eigenvalue}
Lucas Benigni and Sandrine P{\'e}ch{\'e}.
\newblock Eigenvalue distribution of nonlinear models of random matrices.
\newblock {\em arXiv preprint arXiv:1904.03090}, 2019.

\bibitem[CBG16]{couillet2016kernel}
Romain Couillet and Florent Benaych-Georges.
\newblock {Kernel spectral clustering of large dimensional data}.
\newblock {\em Electronic Journal of Statistics}, 10(1):1393--1454, 2016.

\bibitem[CLM18]{couillet2018classification}
Romain Couillet, Zhenyu Liao, and Xiaoyi Mai.
\newblock Classification asymptotics in the random matrix regime.
\newblock In {\em 2018 26th European Signal Processing Conference (EUSIPCO)},
  pages 1875--1879. IEEE, 2018.

\bibitem[CS13]{cheng2013spectrum}
Xiuyuan Cheng and Amit Singer.
\newblock {The spectrum of random inner-product kernel matrices}.
\newblock {\em Random Matrices: Theory and Applications}, 2(04):1350010, 2013.

\bibitem[DV13]{do2013spectrum}
Yen Do and Van Vu.
\newblock The spectrum of random kernel matrices: universality results for
  rough and varying kernels.
\newblock {\em Random Matrices: Theory and Applications}, 2(03):1350005, 2013.

\bibitem[EK10a]{el2010information}
Noureddine El~Karoui.
\newblock {On information plus noise kernel random matrices}.
\newblock {\em The Annals of Statistics}, 38(5):3191--3216, 2010.

\bibitem[EK10b]{el2010spectrum}
Noureddine El~Karoui.
\newblock {The spectrum of kernel random matrices}.
\newblock {\em The Annals of Statistics}, 38(1):1--50, 2010.

\bibitem[FM19]{fan2019spectral}
Zhou Fan and Andrea Montanari.
\newblock The spectral norm of random inner-product kernel matrices.
\newblock {\em Probability Theory and Related Fields}, 173(1-2):27--85, 2019.

\bibitem[HJ12]{horn2012matrix}
Roger~A. Horn and Charles~R. Johnson.
\newblock {\em Matrix Analysis}.
\newblock Cambridge University Press, second edition, 2012.

\bibitem[JL09]{johnstone2009consistency}
Iain~M Johnstone and Arthur~Yu Lu.
\newblock On consistency and sparsity for principal components analysis in high
  dimensions.
\newblock {\em Journal of the American Statistical Association},
  104(486):682--693, 2009.

\bibitem[KMV16]{kalai2016disentangling}
Adam~Tauman Kalai, Ankur Moitra, and Gregory Valiant.
\newblock Disentangling gaussians.
\newblock {\em Communications of the ACM}, 55, 2016.

\bibitem[LC18]{liao2018spectrum}
Zhenyu Liao and Romain Couillet.
\newblock On the spectrum of random features maps of high dimensional data.
\newblock In {\em Proceedings of the 35th International Conference on Machine
  Learning}, volume~80, pages 3063--3071. PMLR, 2018.

\bibitem[LC19]{liao2019large}
Zhenyu Liao and Romain Couillet.
\newblock {A large dimensional analysis of least squares support vector
  machines}.
\newblock {\em IEEE Transactions on Signal Processing}, 67(4):1065--1074, 2019.

\bibitem[LLC18]{louart2018random}
Cosme Louart, Zhenyu Liao, and Romain Couillet.
\newblock A random matrix approach to neural networks.
\newblock {\em The Annals of Applied Probability}, 28(2):1190--1248, 2018.

\bibitem[NJW02]{ng2002spectral}
Andrew~Y Ng, Michael~I Jordan, and Yair Weiss.
\newblock On spectral clustering: Analysis and an algorithm.
\newblock In {\em Advances in neural information processing systems}, pages
  849--856, 2002.

\bibitem[PW17]{pennington2017nonlinear}
Jeffrey Pennington and Pratik Worah.
\newblock Nonlinear random matrix theory for deep learning.
\newblock In {\em Advances in Neural Information Processing Systems}, pages
  2637--2646, 2017.

\bibitem[Rud64]{rudin1964principles}
Walter Rudin.
\newblock {\em Principles of mathematical analysis}, volume~3.
\newblock McGraw-hill New York, 1964.

\bibitem[Wil97]{williams1997computing}
Christopher~KI Williams.
\newblock Computing with infinite networks.
\newblock In {\em Advances in neural information processing systems}, pages
  295--301, 1997.

\end{thebibliography}

\clearpage

\appendix

\vskip 0.3in

\section{The non-trivial classification regime}
\label{sec:sm-non-trivial}
In the ideal case where \(\bmu_1,\bmu_2\) and \(\E_1, \E_2\) are perfectly known, the (decision optimal) Neyman-Pearson test to decide on the class of an unknown and normally distributed \(\x\), genuinely belonging to \(\mathcal{C}_1\), consists in the following comparison
\[
  (\x-\bmu_2)^\T (\I_p + \E_2)^{-1} (\x-\bmu_2) - (\x-\bmu_1)^\T (\I_p + \E_1)^{-1} (\x-\bmu_1) \overset{\mathcal C_1}{\underset{\mathcal C_2}{\gtrless}} \log \frac{\det(\I_p + \E_1)}{\det(\I_p + \E_2)}.
\]
Writing \(\x = \bmu_1 + (\I_p + \E_1)^{\frac12} \z\) so that \(\z \sim \NN(\zo,\I_p)\), the above test is then equivalent to
\begin{align*}
  T(\x) &= \frac1p \z^\T \left( (\I_p + \E_1)^\frac12 (\I_p + \E_2)^{-1} (\I_p + \E_1)^\frac12 - \I_p \right) \z + \frac2p \Delta\bmu^\T (\I_p + \E_2)^{-1} (\I_p + \E_1)^\frac12 \z \\
  %%%
  &+ \frac1p \Delta\bmu^\T (\I_p + \E_2)^{-1} \Delta\bmu - \frac1p \log \frac{\det(\I_p + \E_1)}{\det(\I_p + \E_2)} \overset{\mathcal C_1}{\underset{\mathcal C_2}{\gtrless}} 0
\end{align*}
where \(\Delta \bmu \equiv \bmu_1 - \bmu_2\). Since, for \(\U \in \RR^{p \times p}\) an eigenvector basis for \((\I_p + \E_1)^\frac12 (\I_p + \E_2)^{-1} (\I_p + \E_1)^\frac12 - \I_p\), \(\U \z\) has the same distribution as \(\z\), with a careful application of the Lyapunov's central limit theorem (see for example \cite[Theorem~27.3]{billingsley2012probability}), along with the assumption \(\| \E_a \| = o (1)\) for \(a \in \{1,2\}\), we obtain 
\[
    V_T^{-\frac12} (T(\x) - E_T) \overset{\mathcal L}{\longrightarrow}\NN(0,1)
\]
as \(p \to \infty\), with
\begin{align*}
    E_T &\equiv \frac1p \tr \left( (\I_p + \E_1) (\I_p + \E_2)^{-1} \right) - 1 + \frac1p \Delta\bmu^\T (\I_p + \E_2)^{-1} \Delta\bmu - \frac1p \log \frac{\det(\I_p + \E_1)}{\det(\I_p + \E_2)} \\
    %%%
    V_T &\equiv \frac2{p^2} \| (\I_p + \E_1)^\frac12 (\I_p + \E_2)^{-1} (\I_p + \E_1)^\frac12 - \I_p \|_F^2 \nonumber \\ &+ \frac4{p^2} \Delta\bmu^\T (\I_p + \E_2)^{-1} (\I_p + \E_1) (\I_p + \E_2)^{-1} \Delta\bmu.
\end{align*}
The classification error rate is thus non-trivial (i.e., converging neither to \(0\) not \(1\) as \(p \to \infty\)) if \(E_T\) and \(\sqrt{V_T}\) are of the same order of magnitude (with respect to \(p\)). 

In the case where \(\E_1 = \E_2 = \E\),
\[
    E_T = \frac1p \Delta\bmu^\T (\I_p + \E)^{-1} \Delta\bmu = O( \| \Delta\bmu \|^2 p^{-1} ), \ \sqrt{V_T} = \frac2p \sqrt{ \Delta\bmu^\T (\I_p + \E)^{-1} \Delta\bmu } = O( \| \Delta\bmu \| p^{-1} )
\]
so that we must as least demand \(\| \Delta\bmu \| \ge O(1)\) (which, up to centering, is equivalent to asking \(\| \bmu_a \| \ge O(1)\) for \(a\in\{1,2\}\)).

Under the critical condition \(\| \Delta\bmu \| = O(1)\), we move on to the study of the condition on the covariance \(\E_a\). To this end, a Taylor expansion can be performed for \(\I_p + \E_2\) around \(\I_p + \E_1\) so that
\begin{align*}
    E_T &= \frac1p \Delta\bmu^\T (\I_p + \E_1 )^{-1} \Delta\bmu  + \frac1{2p} \| (\I_p + \E_1 )^{-1} \Delta\E \|_F^2 + o(p^{-1}) \\
    %%%
    V_T &= \frac2{p^2} \| (\I_p + \E_1 )^{-1} \Delta\E \|_F^2 + \frac4{p^2} \Delta\bmu^\T (\I_p + \E_1)^{-1} \Delta\bmu + o(p^{-2}).
\end{align*}
with \(\Delta \E \equiv \E_1 - \E_2\). Thus one must have \(\| \Delta \E \| \ge O(p^{-1/2})\) for \(\| (\I_p + \E_1 )^{-1} \Delta\E \|_F^2 \) not to vanish for \(p\) large and for \(\Delta \E\) to have discriminating power. It is thus convenient to request \(\| \E_a \| \ge O(p^{-1/2})\) for \(a \in \{1,2 \}\), which unfolds from
\[
    |\tr \E_a | \ge O(\sqrt p), \quad\textmd{or}\quad \| \E_a \|_F^2 \ge O(1).
\]

Yet, as noticed in \cite{couillet2018classification}, many classification algorithms, either supervised, semi-supervised or unsupervised, are not able to achieve the optimal rate \(\| \E_a \|_F^2 = O(1)\) when \(n,p\) are of the same order of magnitude. Indeed, the best possible rate \(\| \E_a \|_F^2 = O(\sqrt p)\) can only be obtained in very particular cases, for instance if \(|\tr \E_a | = o(\sqrt p)\) and \(\| \bmu_a \| = o(1)\) as investigated in \cite{liao2019large}. This thus leads to the non-trivial classification condition demanded in Assumption~\ref{ass:growth-rate}.

\section{Exact computation of \texorpdfstring{ \(\phi\)}{something} in the Gaussian case}
\label{sec:sm-expectation-computation-Gaussian}

In the Gaussian case where \(\z \sim \mathcal{N}(\zo,\I_p)\) we resort to computing, as in \cite{williams1997computing,louart2018random}, the integral
\begin{align*}
    &\EE_\z[(\z^\T \a)^{k_1} (\z^\T \b)^{k_2}] = (2\pi)^{-p/2} \int_{\RR^p} (\z^\T \a)^{k_1} (\z^\T \b)^{k_2} e^{-\|\z\|^2/2} d\z \\
    %%%
    &=\frac1{2\pi} \int_{\RR^2} (\tilde z_1 \tilde a_1)^{k_1} (\tilde z_1 \tilde b_1 + \tilde z_2 \tilde b_2)^{k_2} e^{-(\tilde z_1^2 + \tilde z_2^2)/2} d \tilde z_1 d \tilde z_2 = \frac1{2\pi} \int_{\RR^2} (\tilde \z^\T \tilde \a)^{k_1} (\tilde \z^\T \tilde \b)^{k_2} e^{-\|\tilde \z\|^2/2} d \tilde \z
\end{align*}
where we apply the Gram-Schmidt procedure to project \(\z\) into the two-dimensional space\footnote{By assuming first that \(\a,\b\) are linearly independent before extending by continuity to \(\a,\b\) proportional.} spanned by \(\a,\b\) with \(\tilde a_1 = \| \a \|\), \(\tilde b_1 = \frac{\a^\T \b}{\| \a\|}\), \(\tilde b_2 = \sqrt{ \| \b \|^2 - \frac{(\a^\T \b)^2}{\| \a \|^2}}\) and denote \(\tilde \z = [\tilde z_1; \tilde z_2]\), \(\tilde \a = [\tilde a_1; 0]\) and \(\tilde \b = [\tilde b_1; \tilde b_2]\). As a consequence, we obtain, for \(k\) even,
\begin{align*}
    &\EE \left[ (\z_i^\T \z_j/\sqrt{p})^{k}\right]= \EE[\xi^k] = (k - 1)!!; \\
    %%%
    &\EE_{\z_i} \left[ ( \z_i^\T \z_j/\sqrt{p})^{k} (\z_i^\T \b) \right]= \EE_{\z_i} \left[ ( \z_i^\T \z_j/\sqrt{p})^{k-1} (\z_i^\T \b)^2 \right]= 0; \\
    %%%
    &\EE_{\z_i}\left[ ( \z_i^\T \z_j/\sqrt{p})^{k-1} (\z_i^\T \b) \right]= (k - 1)!! (\| \z_j\| /\sqrt{p} )^{k - 2}  (\z_j^\T \b)/\sqrt{p}; \\
    %%%
    & \EE_{\z_i}\left[ ( \z_i^\T \z_j/\sqrt{p})^{k} (\z_i^\T \b)^2 \right]= (k - 1)!! \left( k ( \| \z_j\| /\sqrt{p} )^{k - 2} ( \z_j^\T \b/\sqrt{p})^2 + ( \| \z_j\| /\sqrt{p} )^{k} \| \b \|^2 \right);
\end{align*}
where we denote \(k!!\) the double factorial of an integral \(k\) such that \(k!! = k (k-2) (k-4) \cdots\). This futher leads to, in the Gaussian case, the expression of \(\tilde \K_I\) in Proposition~\ref{prop:K-I-approx}.
\section{Proof of Proposition~\ref{prop:K-I-approx}}
\label{sec:sm-proof-prop-K-I-approx}
%\begin{proof}[Proof of Proposition~\ref{prop:K-I-approx}]
Define by \(\L\) the matrix with \(\L_{ij} \equiv [\frac1p (\J \M^\T \M \J^\T + \J \M^\T \Z + \Z^\T \M \J^\T)]_{ij}\) for \(i\neq j\) and \(\L_{ii} = 0\). Then \(\K_I\) can be written as
\begin{align*}
    \K_I &= k (\Z^\T \Z/\sqrt p)^{\circ (k-1)} \circ \L + \bPhi, \\
    %%%
    \bPhi_{ij} &\equiv \frac{k}p (\z_i^\T \z_j/\sqrt p)^{k-1} \z_i^\T \left( \frac12 (\E_a + \E_b) - \frac18 (\E_a - \E_b)^2 \right) \z_j \\
    %%%
    &+ \frac{k(k-1)}{8p} (\z_i^\T \z_j/\sqrt p)^{k-2} \frac1{\sqrt p} (  \z_i^\T (\E_a + \E_b) \z_j )^2
\end{align*}
for \(i\neq j\) and \(\bPhi_{ii} = 0\). With this expression, the proof of Proposition~\ref{prop:K-I-approx} can be divided into three steps.

\paragraph{Concentration of \(\bPhi\).}

We first show that, \(\| \bPhi - \EE[\bPhi] \| \to 0\) almost surely, as \(n,p \to \infty\). This follows from the observation that \(\bPhi\) is a \(p^{-1/4}\) rescaling (since \(\| \E_a \| = O(p^{-1/4})\)) of the null model \(\K_N\), which concentrates around its expectation in the sense that \(\| \K_N - \EE[\K_N] \|= O(1)\) for \(\EE[\K_N] = O(\sqrt p)\) if \(a_0 \neq 0\) (see Remark~\ref{rem:a_0}). Indeed, it is shown in \cite[Theorem~1.7]{fan2019spectral} that, the leading eigenvalue of order \(O(\sqrt p)\) discarded (arising from \(\EE[\K_N]\)), \(\K_N\) is of bounded operator norm for all large \(n,p\) with probability one; this, together with the fact that \(\| \EE[\bPhi] \| = O(1)\) that will be shown subsequently, allows us to conclude that \(\| \bPhi - \EE[\bPhi] \| \to 0\) as \(n,p \to \infty\).

\paragraph{Computation of \(\EE[\bPhi]\): beyond the Gaussian case.}

We then show that, for random vectors \(\z\) with zero mean, unit variance and bounded moments entries, the expression of \(\EE[\bPhi]\) coincides with the Gaussian case. To this end, recall that the entries of \(\bPhi\) are the sum of random variables of the type
\[
  \phi = \frac{C}{\sqrt p} (\x^\T \y /\sqrt p)^\alpha (\x^\T \F \y)^\beta
\]
for independent random vectors \(\x,\y \in \RR^p\) with i.i.d.\@ zero mean, unit variance and finite moments (uniformly on \(p\)) entries, deterministic \(\F \in \RR^{p \times p}\), \(C\in\RR\), \(\alpha\in \mathbb N\) and \(\beta\in\{1,2\}\). Let us start with the case \(\beta = 1\) and expand \(\phi\) as
\begin{equation}\label{eq:phi-expension-1}
  \phi = \frac{C}{\sqrt p} \left( \frac1{\sqrt p} \sum_{i_1=1}^p x_{i_1} y_{i_1} \right) \ldots \left( \frac1{\sqrt p} \sum_{i_\alpha=1}^p x_{i_\alpha} y_{i_\alpha} \right) \left( \sum_{j_1, j_2=1}^p F_{j_1,j_2} x_{j_1} y_{j_2} \right)
\end{equation}
with \(x_i\) and \(y_i\) the \(i\)-th entry of \(\x\) and \(\y\), respectively, so that (i) \(x_i\) is independent of \(y_j\) for all \(i,j\) and (ii) \(x_i\) is independent of \(x_j\) for \(i \neq j\) with \(\EE[x_i] = 0\), \(\EE[x_i^2] = 1\) and \(\EE[|x_i|^k] \le C_k\) for some \(C_k\) independent of \(p\) (and similarly for \(\y\)).

At this point, note that to ensure \(\EE[\K_I]\) has non-vanishing operator norm as \(n,p \to \infty\), we need \(\EE[\phi] \geq O(p^{-1})\) since \( \| \A \| \leq p \| \A \|_\infty \) for \(\A \in \RR^{p \times p}\). Also, note that (as \(\beta = 1\)), all terms in the sum \(\sum_{j_1, j_2=1}^p F_{j_1,j_2} x_{j_1} y_{j_2}\) with \(j_1 \neq j_2\) must be zero since in other terms \(x_i\) always appears together with \(y_i\), so that all terms with \(j_1 \neq j_2\) give rise to zero in expectation. Hence, the \(p^2\) terms of the sum only contain \(p\) nonzero terms in expectation (those with \(j_1 = j_2\)). The arbitrary (absolute) moments of \(x\) and \(y\) being finite, the first \(\alpha p\) terms must be divided into \(\lceil \alpha \rceil/2\) groups of size two (containing \(O(p)\) terms) so that, with the normalization by \(p^{-1}\) for each group of size two, the associated expectation is not vanishing. We shall thus discuss the following two cases:
\begin{enumerate}
  \item \(\alpha\) even: the \(\alpha\) terms in the sum form \(\alpha/2\) groups with different indices each and also different from \(j_1 = j_2\). Therefore we have \(\EE_{x_j}[\phi]= 0\) and \(\EE[\phi] = 0\).
  \item \(\alpha\) odd: the \(\alpha\) terms in the sum form \((\alpha - 1)/2\) groups with indices different from each other and the remaining one goes with the last term containing \(\F\) and one has \(\EE[\phi] = \frac{C \alpha!!}p \tr (\F)\) by a simple combinatorial argument.
\end{enumerate}
The case \(\beta = 2\) follows exactly the same line of arguments except that \(j_1\) may not equal \(j_2\) to give rise to non-vanishing terms.  

\paragraph{Concentration of Hadamard product.}

It now remains to treat the term \(k(\Z^\T \Z/\sqrt p)^{\circ (k-1)} \circ \L\) and show it also has an asymptotically deterministic behavior (as \(\bPhi\)). It can be shown that
\[
    \| \N \circ \L \| \to 0, \quad n,p \to \infty
\]
with \(\N = (\Z^\T \Z/ \sqrt p)^{\odot (k-1)} - (k-2)!! \one_n \one_n^\T\) for \(k\) odd and \(\N = (\Z^\T \Z/ \sqrt p)^{\odot (k-1)}\) for \(k\) even.

To prove this, note that, depending on the key parameter \(a_0 = \EE[f(\xi)]\), the operator norm of \(f(\Z^\T \Z/\sqrt p)/\sqrt p\) is either of order \(O(\sqrt p)\) for \(a_0 \neq 0\) or \(O(1)\) for \(a_0 = 0\). In particular, for monomial \(f(x) = x^k\) under study here, we have \(a_ 0 = \EE[\xi^k] = 0\) for \(k\) odd and \(a_0 = \EE[\xi^k] = (k-1)!! \neq 0\) for \(k\) even, \(\xi \sim \NN(0,1)\). To control the operator norm of the Hadamard product between matrices, we introduce the following lemma.

\begin{Lemma}\label{lem:hadamard-product-control}
For \(\A,\B \in \mathbb{R}^{p \times p}\), we have \(\| \A \circ \B \| \le \sqrt{p} \| \A \|_\infty \| \B\|\).
\end{Lemma}
\begin{proof}[Proof of Lemma~\ref{lem:hadamard-product-control}]
Let \(\e_1, \ldots, \e_p\) be the canonical basis vectors of \(\mathbb{R}^p\), then for all \(1 \le i \le p\),
\[
  \| (\A \circ \B) \e_i \| \le \max_{i,j} |\A_{ij}| \| \B \e_i \| = \| \A \|_\infty \|\B \e_i\| \le \| \A \|_\infty \|\B \|.
\]
As a consequence, for any \(\v = \sum_{i=1}^p v_i \e_i\), we obtain
\[
  \| (\A \circ \B) \v \| \le \sum_{i=1}^p |v_i| \| (\A \circ \B) \e_i \|  \le \sum_{i=1}^p |v_i| \| \A \|_\infty \|\B\|
\]
which, by Cauchy-Schwarz inequality further yields \(\sum_{i=1}^p |v_i| \le \sqrt{p} \| \v\|\). This concludes the proof of Lemma~\ref{lem:hadamard-product-control}.
\end{proof}
Lemma~\ref{lem:hadamard-product-control} tells us that the Hadamard product between a matrix with \(o(p^{-1/2})\) entry and a matrix with bounded operator norm is of vanishing operator norm, as \(p \to \infty\). As such, since \(\| \N \| = O(1)\) and \(\L\) has \(O(p^{-1})\) entries, we have \(\| \N \circ \L \| \to 0\). This concludes the proof of Proposition~\ref{prop:K-I-approx}.
%\end{proof}
%
%
%
%
%
\section{Proof of Theorem~\ref{prop:K-approx-final}}
\label{sec:sm-proof-prop-K-approx-final}

%\begin{proof}[Proof of Theorem~\ref{prop:K-approx-final}]
The proof follows from the fact that the individual coefficients of the Hermite polynomials \(P_\kappa(x) = \sum_{l=0}^\kappa c_{\kappa,l} x^l\) satisfy the following recurrent relation \cite{abramowitz1965handbook}
\begin{equation}\label{eq:hermite-coeff}
  c_{\kappa+1,l} = \begin{cases} - \kappa c_{\kappa-1,l} & l=0; \\ c_{\kappa,l-1} - \kappa c_{\kappa-1,l} & l \ge 1; \end{cases}
\end{equation}
with \(c_{0,0} = 1\), \(c_{1,0} = 0\) and \(c_{1,1} = 1\). As a consequence, by indexing the informative matrix in Proposition~\ref{prop:K-I-approx} of the monomial \(f(x) = x^l\) as \(\tilde\K_{I,l}\), we have for odd \(\kappa \ge 3\),
\[
  \tilde\K_I = \sum_{l=1,3,\ldots}^\kappa c_{\kappa,l} \tilde\K_{I,l} = \sum_{l=1,3,\ldots}^\kappa c_{\kappa,l} l!! (\J \M^\T \M \J^\T + \J \M^\T \Z + \Z^\T \M \J^\T)/p -\diag(\cdot) = \zo
\]
with \([\X-\diag(\cdot)]_{ij}=\X_{ij}\delta_{i\neq j}\).
This follows from the fact that, for \(\kappa \ge 3\), we have both \(\sum_{l=1,3,\ldots}^\kappa c_{\kappa,l} l!! = 0\) and \(\sum_{l=0,2,\ldots}^{\kappa+1} c_{\kappa+1, l} (l+1)!! = 0\). The latter is proved by induction on \(\kappa\): first, for \(\kappa = 3\), we have \(c_{3,1} + 3 c_{3,3} = c_{4,0} + 3 c_{4,2} + 15 c_{4,4} = 0\); then, assuming \(\kappa\) odd, we have \(\sum_{l=1,3,\ldots}^\kappa c_{\kappa,l} l!! = \sum_{l=0,2,\ldots}^{\kappa+1} c_{\kappa+1,l} (l+1)!! = 0\) so that, together with \eqref{eq:hermite-coeff}
\[
  \sum_{l=1,3,\ldots}^{\kappa+2} c_{\kappa+2,l} l!! = \sum_{l=1,\ldots}^{\kappa+2} ( c_{\kappa+1,l-1} - (\kappa+1) c_{\kappa,l}) l!! = \sum_{l=1,\ldots}^{\kappa+2} c_{\kappa+1,l-1} l!! = \sum_{l=0,2,\ldots}^{\kappa+1} c_{\kappa+1, l} (l+1)!! = 0
\]
as well as
\begin{align*}
  \sum_{l=0,2,\ldots}^{\kappa+3} c_{\kappa+3,l} (l+1)!! &= - (\kappa+2) c_{\kappa+1,0} + \sum_{l=2,4,\ldots}^{\kappa+3} (c_{\kappa+2, l-1} - (\kappa+2) c_{\kappa+1,l}) (l+1)!! = 0
\end{align*}
where we used \(c_{\kappa,l} = 0\) for \(l \ge \kappa+1\). Similar arguments hold for the case of \(\kappa\) even, which concludes the proof.
%\end{proof}

% {\RED 
% \emph{Note on simulations}:

% Consider a symmetric Toeplitz matrix \(\A \in \RR^{p \times p}\) such that \(\A_{ij} = a^{|i-j|}\) for some \(0 < a <1\), we have the upper bound \(\| \A \| \le \max_{1 \le i \le p} \sum_{j=1}^p |\A_{ij}| \le \frac{1+a}{1-a}\) as a results of the Gershgorin circle theorem (e.g., Theorem~6.1.1 in \cite{horn2012matrix}). As a consequence, for \(\E\) to have operator norm of order \(O(p^{-1/4})\), one may wish to take for instance \(\E = p^{-\frac14} \A\) for \(0 < a < 1\).

% Then, since the (squared) Frobenius norm of \(\A\) is given by 
% \[
%     \| \A \|_F^2 = p (c^2)^0 + 2 (p-1) (c^2)^1 +  2 (p-2) (c^2)^2 + \ldots + 2 (c^2)^{p-1} \le \frac{1+a^2}{1-a^2} p %- \frac{2 a^2}{(1-a^2)^2}
% \]
% so that \(\|\E\|_F^2 = O(\sqrt p) \) as required.

% Also, for \(\| \I_p - \A\|_F^2 = 2 (p-1) (a^2)^1 +  2 (p-2) (a^2)^2 + \ldots +2 (a^2)^{p-1}\) to ensure
% \[
%   \| \I_p - \A\|_F^2 = \frac{2 a^2 p}{1-a^2} - \frac{2a^2}{(1-a^2)^2} - 2c^{2p} \frac{2-a^2}{(1-a^2)^2} = O(\sqrt{p})
% \]
% it suffices to take \(a = p^{-\alpha}\) with \(\alpha = \frac14\).
% }

\end{document}